\documentclass{article}

\usepackage{microtype}
\usepackage{graphicx}
\usepackage{subfigure}
\usepackage{enumitem}
\usepackage{booktabs} 

\usepackage{hyperref}

\usepackage{natbib}
\usepackage{amsmath}
\usepackage{amsthm}
\usepackage{amssymb}
\usepackage{mathtools}
\usepackage{tikz}
\usepackage{xcolor}
\usetikzlibrary{arrows}
\usepackage{comment}

\DeclareMathOperator*{\argmin}{\mathrm{argmin}}

\newcommand{\diag}{\mat{diag}}

\newcommand{\unif}{U}

\newcommand{\holder}{H\"{o}lder }

\def\R{\mathbb{R}}

\def\cE{\mathcal{E}}
\def\cF{\mathcal{F}}
\def\cG{\mathcal{G}}

\def\cM{\mathcal{M}}

\def\cZ{\mathcal{Z}}

\newcommand{\card}[1]{\lvert#1\rvert}

\newcommand{\set}[1]{\{#1\}}

\newcommand{\bigBracks}[1]{\bigl[#1\bigr]}

\newcommand{\Bracks}[1]{\left[#1\right]}

\newcommand{\bigParens}[1]{\bigl(#1\bigr)}
\newcommand{\BigParens}[1]{\Bigl(#1\Bigr)}

\newcommand{\given}{\mathbin{\vert}}
\newcommand{\bigGiven}{\mathbin{\bigm\vert}}

\newcommand{\norm}[1]{\left\|#1\right\|}

\newcommand{\bigNorm}[1]{\bigl\lVert#1\bigr\rVert}

\newcommand{\abs}[1]{\left|#1\right|}

\newcommand{\BigAbs}[1]{\Bigl\lvert#1\Bigr\rvert}

\newcommand{\mat}[1]{\mathbf{#1}}
\newcommand{\vect}[1]{\mathbf{#1}}
\newcommand{\expect}{\mathbb{E}}
\newcommand{\prob}{\mathbb{P}}

\newcommand{\gclass}{\mathcal{G}}

\newcommand{\states}{\mathcal{S}}

\newcommand{\actions}{\mathcal{A}}
\newcommand{\contexts}{\mathcal{X}}

\newcommand{\simplex}{\triangle}

\newtheorem{thm}{Theorem}[section]
\newtheorem{lem}{Lemma}[section]
\newtheorem{cor}{Corollary}[section]

\newtheorem{asmp}{Assumption}[section]
\newtheorem{defn}{Definition}[section]
\newtheorem{claim}{Claim}[section]

\newtheorem{condition}{Condition}[section]
\newtheorem{corollary}{Corollary}[section]

\usepackage{xspace}

\newcommand{\oracleq}{\textsc{OracleQ}\xspace}
\newcommand{\qlearning}{\textsc{QLearning}\xspace}
\newcommand{\decoding}{\textsc{PCID}\xspace}
\newcommand{\hats}{\ensuremath{\hat{s}}}
\newcommand{\hatset}{\ensuremath{\widehat{\states}}}
\newcommand{\hatf}{\ensuremath{\hat{f}}}
\newcommand{\emb}{\boldsymbol{\phi}}
\newcommand{\embh}{\widehat{\emb\,}\!}
\newcommand{\decodeTrue}{f^*}

\newcommand{\vp}{\mathbf{p}}
\newcommand{\vq}{\mathbf{q}}
\newcommand{\p}{p}
\newcommand{\hvp}{\widehat{\vp}}
\newcommand{\q}{q}

\newcommand{\vb}{\mathbf{b}}
\newcommand{\vg}{\mathbf{g}}
\newcommand{\ta}{\tilde{a}}
\newcommand{\ts}{\tilde{s}}
\newcommand{\vy}{\mathbf{y}}
\newcommand{\hvg}{\hat{\vg}}
\newcommand{\hp}{\hat{p}}
\newcommand{\hf}{\hat{f}}
\newcommand{\epsp}{\epsilon_{\textup{p}}}
\newcommand{\epsdec}{\epsilon_{\textup{f}}}
\newcommand{\ve}{\mathbf{e}}
\newcommand{\tO}{\tilde{O}}
\newcommand{\tOmega}{\tilde{\Omega}}
\newcommand{\tTheta}{\tilde{\Theta}}
\newcommand{\Nexp}{N_{\textup{g}}}
\newcommand{\Dexp}{D_{\textup{g}}}
\newcommand{\Ncluster}{N_{\phi}}
\newcommand{\Np}{N_{\textup{p}}}
\newcommand{\Dp}{D_{\textup{p}}}
\newcommand{\Nboost}{N_{\textup{b}}}
\newcommand{\Dboost}{D_{\textup{b}}}
\newcommand{\hvz}{\widehat{\mathbf{z}}}
\newcommand{\vz}{\mathbf{z}}
\newcommand{\vv}{\mathbf{v}}

\newcommand{\tnu}{\tilde{\nu}}
\newcommand{\vx}{\mathbf{x}}
\newcommand{\mW}{\mathbf{W}}
\newcommand{\vc}{\mathbf{c}}

\newcommand{\hatb}{\hat{b}}
\newcommand{\hvb}{\hat{\vb}}
\newcommand{\vQ}{\mathbf{Q}}

\usepackage[accepted]{icml2019}

\icmltitlerunning{Provably efficient RL with Rich Observations via Latent State Decoding}

\allowdisplaybreaks

\newenvironment{itemize*}%
  {\begin{itemize}[leftmargin=*,topsep=0pt]%
    \setlength{\itemsep}{0pt}%
    \setlength{\parskip}{0pt}}%
  {\end{itemize}}

\begin{document}

\twocolumn[
\icmltitle{Provably efficient RL with Rich Observations via Latent State Decoding}




\begin{icmlauthorlist}
	\icmlauthor{Simon S. Du}{cmu}
	\icmlauthor{Akshay Krishnamurthy}{msrnyc}
	\icmlauthor{Nan Jiang}{uiuc}
	\icmlauthor{Alekh Agarwal}{msrred}
	\icmlauthor{Miroslav Dud\'{i}k}{msrnyc}
	\icmlauthor{John Langford}{msrnyc}

\end{icmlauthorlist}

\icmlaffiliation{cmu}{Carnegie Mellon University}
\icmlaffiliation{msrnyc}{Microsoft Research, New York}
\icmlaffiliation{uiuc}{University of Illinois at Urbana-Champaign}
\icmlaffiliation{msrred}{Microsoft Research,
	Redmond}
\icmlcorrespondingauthor{Simon S. Du}{ssdu@cs.cmu.edu}

\icmlkeywords{exploration}

\vskip 0.3in
]



\printAffiliationsAndNotice{}  
\begin{abstract}
	We study the exploration problem in episodic MDPs with rich observations generated from a small number of latent states. Under certain identifiability assumptions, we demonstrate how to estimate a mapping from the observations to latent states inductively through a sequence of regression and clustering steps---where previously decoded latent states provide labels for later regression problems---and use it to construct good exploration policies. We provide finite-sample guarantees on the quality of the learned state decoding function and exploration policies, and complement our theory with an empirical evaluation on a class of hard exploration problems.
Our method exponentially improves over $Q$-learning with na\"ive exploration, even when $Q$-learning has cheating access to latent states.

\end{abstract}

\setlength{\textfloatsep}{10pt plus 2pt minus 2pt}

\section{Introduction}
\label{sec:intro}
We study reinforcement learning (RL) in episodic environments with rich  observations, such as images and texts.
While many modern empirical RL algorithms are designed to handle such settings~\citep[see, e.g.,][]{mnih2015human},
only few works study how to explore well in these environments~\citep{ostrovski2017count,osband2016deep}
and the sample efficiency of these techniques is not theoretically understood.\looseness=-1
%

From a theoretical perspective, strategic exploration algorithms for provably sample-efficient RL have long existed in the classical tabular setting \citep{kearns2002near, brafman2002r}. However, these methods are difficult to adapt to rich observation spaces, because they all require a number of interactions polynomial in the number of \emph{observed states},
and, without additional structural assumptions, such a dependency is unavoidable \citep[see, e.g.,][]{jaksch2010near, lattimore2012pac}.
Consequently, treating the observations directly as unique states makes this class of methods unsuitable for most settings of practical interest.

In order to avoid the dependency on the observation space, one must exploit some inherent structure in the problem.
The recent line of work on contextual decision processes~\citep{krishnamurthy2016pac,jiang2017contextual,dann2018polynomial} identified certain low-rank structures that enable exploration algorithms with sample complexity polynomial in the rank parameter. 
Such low-rank structure is crucial to circumventing information-theoretic hardness, and is typically found in problems where complex observations are emitted from a small number of \emph{latent states}.
Unlike tabular approaches, which require the number of states to be small and observed, these works are able to handle settings where the observation spaces are uncountably large or continuous and the underlying states never observed during learning.
They achieve this by exploiting the low-rank structure implicitly, operating only in the observation space. The resulting algorithms are sample-efficient, but either provably computationally intractable, or practically quite cumbersome even under strong assumptions~\citep{dann2018polynomial}.

In this work, we take an alternative route: we recover the latent-state structure explicitly
by learning a \emph{decoding function} (from a large set of candidates) that maps a rich observation to the corresponding latent state; note that if such a function is learned perfectly, the rich-observation problem is reduced to a tabular problem where exploration is tractable. We show that our algorithms are:

\emph{Provably sample-efficient:} Under certain identifiability assumptions, we recover a mapping from the observations to underlying latent states as well as a good exploration policy using a number of samples which is polynomial in the number of latent states, horizon and the complexity of the decoding function class with no explicit dependence on the observation space size. Thus we significantly generalize beyond the works of~\citet{dann2018polynomial} who require deterministic dynamics and~\citet{azizzadenesheli2016pomdp} whose guarantees scale with the observation space size.

\emph{Computationally practical:} Unlike many prior works in this vein, our algorithm is easy to implement and substantially outperforms na\"ive exploration in experiments, even when the baselines have cheating access to the latent states.

In the process, we introduce a formalism called \emph{block Mar\-kov decision process} (also implicit in some prior works), and a new solution concept for exploration called $\epsilon$--policy cover.\looseness=-1

The main challenge in learning the decoding function is that the hidden states are never directly observed.
Our key novelty is the use of a backward conditional probability vector (Equation~\ref{eqn:bnu}) as a representation for latent state, and learning the decoding function via conditional probability estimation, which can be solved using least squares regression.
While learning a low-dimensional representations of rich observations has been explored in recent empirical works \citep[e.g.,][]{silver2017predictron, oh2017value, pathak2017curiosity}, our work provides a precise mathematical characterization of the structures needed for such approaches to succeed and comes with rigorous sample-complexity guarantees.




\section{Setting and Task Definition}
\label{sec:pre}
We begin by introducing some basic notation. We write $[h]$ to denote the set
$\left\{1,\ldots,h\right\}$.
For any finite set $S$, we write $U(S)$ to denote the uniform distribution over $S$.
We write $\simplex_d$ for the simplex in $\R^d$. Finally, we write
$\norm{\cdot}$ and $\norm{\cdot}_1$, respectively, for the Euclidean and the $\ell_1$ norms of a vector.

\subsection{Block Markov Decision Process}
\label{sec:cdp}

In this paper we introduce and analyze a \emph{block Markov decision process} or \emph{BMDP}.
It refers to an environment described by a
finite, but unobservable \emph{latent state space} $\states$, a finite \emph{action space} $\actions$, with $\card{\actions}=K$,
and a possibly infinite, but observable \emph{context space} $\contexts$. The dynamics of a BMDP is described by the \emph{initial state} $s_1\in\states$ and two conditional probability functions:
the \emph{state-transition function} $p$ and \emph{context-emission function} $q$, defining conditional probabilities $p(s'\given s,a)$ and $q(x\given s)$ for all $s,s'\in\states$, $a\in\actions$, $x\in\contexts$.\footnote{%
 For continuous context spaces, $q(\cdot\given s)$ describes a density function relative to a suitable measure (e.g., Lebesgue measure).}

The model may further include a distribution of reward conditioned on context and action. However, rewards do not play a role in the central task of the paper, which is the exploration of all latent states. Therefore, we omit rewards from our formalism, but we discuss in a few places how our techniques apply in the presence of rewards (for a thorough discussion see Appendix~\ref{app:rewards}).

We consider episodic learning tasks with a finite horizon $H$. In each episode, the environment starts in the state $s_1$. In the step $h\in[H]$ of an episode, the environment generates a context $x_h\sim\q(\cdot\given s_h)$, the agent observes the context $x_h$ (but not the state $s_h$), takes an action $a_h$, and the environment transitions to a new state $s_{h+1}\sim p(\cdot\given s_h,a_h)$. The sequence $(s_1,x_1,a_1,\dotsc,s_H,x_H,a_H,s_{H+1},x_{H+1})$ generated in an episode is called a \emph{trajectory}. We emphasize that a learning agent does not observe components $s_h$ from the trajectory.

So far, our description resembles that of a partially observable Markov decision process (POMDP). To finish the definition of BMDP, and distinguish it from a POMDP, we make the following assumption:
\begin{asmp}[Block structure]
\label{asmp:block}
Each context $x$ uniquely determines its generating state $s$. That is, the context space $\contexts$ can be partitioned into disjoint blocks $\contexts_s$, each containing the support of the conditional distribution $q(\cdot\given s)$.
\end{asmp}

The sets $\contexts_s$ are unique up to the sets of measure zero under $q(\cdot\given s)$. In the paper, we say ``\emph{for all $x\in\contexts_s$}'' to mean ``\emph{for all $x\in\contexts_s$ up to a set of measure zero under $q(\cdot\given s)$}.''\looseness=-1

The block structure implies the existence of a \emph{perfect decoding function} $\decodeTrue:\contexts\to\states$, which maps contexts into their generating states.
This means that a BMDP is indeed an MDP with the transition
operator $P(x'\given x,a)=\q\bigParens{x'\given\decodeTrue(x')}\p\bigParens{\decodeTrue(x')\bigGiven \decodeTrue(x),a}$.
Hence 
the contexts $x$ observed by the agent form valid Markovian states, but the size of $\contexts$ is too large, so  only learning the MDP parameters in the smaller, latent space $\states$ is tractable.

The BMDP model is assumed in several prior works~\citep[e.g.,][]{krishnamurthy2016pac,azizzadenesheli2016pomdp,dann2018polynomial}, without the explicit name. It naturally captures
visual grid-world environments studied in empirical RL~\citep[e.g.,][]{johnson2016malmo}, and also models noisy observations of the latent state due to imperfect sensors.
While the block-structure assumption appears severe, it is necessary for efficient learning if the reward is allowed to depend arbitrarily on the latent state (cf.~Propositions 1 and 2 of \citealt{krishnamurthy2016pac}). In our experiments, we study the robustness of our algorithms to this assumption.


To streamline our analysis, we make a standard assumption for episodic settings. We assume that 
$\states$ can be partitioned into disjoints sets $\states_h$, $h\in[H+1]$, such that $p(\cdot\given s,a)$ is supported on $\states_{h+1}$ whenever $s\in\states_h$. We refer to $h$ 
as the \emph{level} and assume that it is observable as part of the context, so the context space is also partitioned into sets $\contexts_h$. We use notation $\states_{[h]}=\cup_{\ell\in[h]}\states_\ell$ for the set of states up to level~$h$, and similarly define $\contexts_{[h]}=\cup_{\ell\in[h]}\contexts_\ell$.

We assume that $\card{\states_h}\le M$. We seek learning algorithms that scale polynomially in parameters $M$, $K$
and $H$, but do not explicitly depend on $\card{\contexts}$, which might be infinite.

\subsection{Solution Concept: Cover of Exploratory Policies}
\label{sec:solution}


In this paper, we focus on the problem of exploration. Specifically, for each state $s\in\states$, we seek an agent strategy for reaching that state $s$. We formalize an agent strategy as an \emph{$h$-step policy},
which is a map $\pi:\contexts_{[h]}\to\actions$ specifying which action to take in each context up to step $h$.
When executing an $h$-step policy $\pi$ with $h<H$, an agent acts according to $\pi$ for $h$ steps and then arbitrarily until the end of the episode (e.g., according to a specific default policy).

For an $h$-step policy $\pi$, we write $\prob^{\pi}$ to denote the probability distribution over $h$-step trajectories induced by $\pi$. We write $\prob^\pi(\cE)$ for the probability of an event $\cE$. For example,
$\prob^{\pi}(s)$ is the probability of reaching the state~$s$ when executing~$\pi$.\looseness=-1

We also consider randomized strategies, which we formalize as \emph{policy mixtures}. An $h$-step \emph{policy mixture} $\eta$ is a distribution over $h$-step policies. When executing $\eta$, an agent randomly draws a policy $\pi \sim \eta$ at the beginning of the episode, and then follows $\pi$ throughout the episode. The induced distribution over $h$-step trajectories is denoted $\prob^\eta$.

Our algorithms create specific policies and policy mixtures via concatenation. Specifically, given an $h$-step policy $\pi$, we write $\pi\odot a$ for the $(h+1)$-step policy that executes $\pi$ for $h$ steps and chooses action $a$ in step $h+1$. Similarly, if $\eta$ is a policy mixture and $\nu$ a distribution over $\actions$, we write $\eta\odot\nu$ for the policy mixture equivalent to first sampling and following a policy according to $\eta$ and then independently sampling and following an action according to $\nu$.

We finally introduce two key concepts related to exploration: \emph{maximum reaching probability} and \emph{policy cover}.
\begin{defn}[Maximum reaching probability.]
\label{defn:max_reach_prob}
For any $s\in\states$, its \emph{maximum reaching probability} $\mu(s)$ is
\[\textstyle
    \mu(s)
	:= \smash{\max_{\pi}\prob^\pi(s)},
\]
where the maximum is taken over all maps $\contexts_{[H]}\to\actions$. The policy
attaining the maximum for a given $s$ is denoted $\pi_s^*$.\footnote{It suffices to consider maps $\contexts_{[h]}\to \actions$ for $s \in \states_{h+1}$.}
\end{defn}
%
%
Without loss of generality, we assume that all the states are reachable, i.e., $\mu(s)>0$ for all $s$.
We write $\mu_{\min}=\min_{s\in\states}\mu(s)$ for the $\mu(s)$ value of the hardest-to-reach state. Since $\states$ is finite and all states are reachable, $\mu_{\min}>0$.\looseness=-1


Given maximum reaching probabilities, we formalize the task of finding policies that reach states $s$ as the task of finding an \emph{$\epsilon$--policy cover} in the following sense:
\begin{defn}[Policy cover of the state space]
\label{defn:cover}
We say that a set of policies $\Pi_h$ is an \emph{$\epsilon$--policy cover} of $\states_h$ if for all $s\in\states_h$ there exists an $(h-1)$-step policy $\pi\in\Pi_h$ such that $\prob^\pi(s) \geq \mu(s) - \epsilon$. A set of policies $\Pi$ is an \emph{$\epsilon$--policy cover} of $\states$ if it is an $\epsilon$--policy cover of $\states_h$ for all $h\in[H+1]$.
\end{defn}
Intuitively, we seek a policy cover of a small size, typically $O(\card{\states})$, and with a small $\epsilon$.
Given such a cover, we can reach every state with the largest possible probability (up to $\epsilon$) by executing each policy from the cover in turn. This enables us to collect a dataset of observations and rewards at all (sufficiently) reachable states $s$ and further obtain a policy that maximizes any reward (details in Appendix~\ref{app:rewards}).

\section{Embedding Approach}
\label{sec:embedding}


A key challenge in solving the BMDP exploration problem is the lack of access to the latent state $s$. Our algorithms work by explicitly learning a decoding function $f$ which maps contexts to the corresponding latent states. This appears to be a hard unsupervised learning problem, even under the block-structure assumption, unless we make strong assumptions about the structure of $\contexts_s$ or about the emission distributions $q(\cdot\given s)$. Here, instead of making assumptions about $q$ or $\contexts_s$, we make certain ``separability'' assumptions about the latent transition probabilities $p$. Thus, we retain a broad flexibility to model rich context spaces, and also obtain the ability to efficiently learn a decoding function~$f$.
In this section, we define key components of our approach and formally state the separability assumption.
\looseness=-1

\subsection{Embeddings and Function Approximation}

In order to construct the decoding function $f$, we learn low-dimensional representations of contexts as well as latent states in a shared space, namely $\Delta_{MK}$. We learn embedding functions $\vg:\contexts\to\Delta_{MK}$ for contexts and $\emb:\states\to\Delta_{MK}$ for states, with the goal that $\vg(x)$ and $\emb(s)$ should be close if and only if $x\in\contexts_s$. Such embedding functions always exist due to the block-structure: for any set of distinct vectors $\set{\emb(s)}_{s\in\states}$, it suffices to define $\vg(x)=\emb(s)$ for $x\in\contexts_s$.

As we see later in this section, embedding functions $\emb$ and $\vg$ can be constructed via an essentially supervised approach, assuming separability. The state embedding $\emb$ is a lower complexity object (a tuple of at most $\card{\states}$ points in $\Delta_{MK}$), whereas the context embedding $\vg$ has a high complexity for even moderately rich context spaces. Therefore, as is standard in supervised learning, we limit attention to functions $\vg$ from some class $\gclass \subseteq \set{\contexts \to \Delta_{MK}}$, such as generalized linear models, tree ensembles, or neural nets. This is a form of function approximation where the choice of $\gclass$ includes any inductive biases about the structure of the contexts.
By limiting the richness of $\gclass$, we can generalize across contexts as well as control the sample complexity of learning. At the same time, $\gclass$ needs to include embedding functions that reflect the block structure. 
Allowing a separate $\vg_h\in\gclass$ for each level, we require realizability in the following sense:
\begin{asmp}[Realizability]
\label{asmp:realizability}
	For any $h \in [H+1]$ and $\emb: \states_h \rightarrow \simplex_{MK}$,
    there exists $\vg_h \in \gclass$ such that $\vg_h(x)  = \emb(s)$ for all $x\in\contexts_s$ and $s\in\states_h$.
\end{asmp}
In words, the class $\gclass$ must be able to match any state-embedding function $\emb$ across all blocks $\contexts_s$. To satisfy this assumption, it is natural to consider classes $\gclass$ obtained via a composition $\emb'\circ f$ where $f$ is a decoding function from some class $\cF\subseteq\set{\contexts\to\states}$ and $\emb'$ is any mapping $\states\to\simplex_{MK}$. Conceptually, $f$ first decodes the context $x$ to a state $f(x)$ which is then embedded by $\emb'$
into $\simplex_{MK}$.
The realizability assumption is satisfied as long as $\cF$ contains a perfect decoding function $f^*$, for which $f^*(x) = s$ whenever $x\in\contexts_s$. The core representational power of $\gclass$ is thus driven by $\cF$, the
class of candidate decoding functions $f$.




Given such a class $\gclass$, our goal is find a suitable context-embedding function in $\gclass$ using a number of trajectories that is proportional to $\log\,\card{\gclass}$ when $\gclass$ is finite, or a more general notion of complexity such as a log covering number when $\gclass$ is infinite. Throughout this paper, we assume that $\gclass$ is finite as it serves to illustrate the key ideas, but our approach generalizes to the infinite case using standard techniques.

As we alluded to earlier, we learn context embeddings $\vg_h$ by solving supervised learning problems. In fact,
we only require the ability to solve least squares problems. Specifically, we assume access to an algorithm for
solving vector-valued least-squares regression over the class $\gclass$. We refer to such an algorithm as the ERM oracle:
\begin{defn}[ERM Oracle]
	\label{defn:weighted_erm}
	Let $\gclass$ be a function class that maps $\contexts$ to $\simplex_{MK}$.
    An \emph{empirical risk minimization oracle (ERM oracle)} for $\gclass$ is any algorithm that
    takes as input a data set $D = \set{(x_i,\vy_i)}_{i=1}^n$ with $x_i \in \contexts$, $\vy_i \in  \simplex_{MK}$, and 
    computes $\argmin_{\vg \in \gclass} \sum_{(x,\vy)\in D} \norm{\vg(x)-\vy}^2$.
\end{defn}
%

\subsection{Backward Probability Vectors and Separability}

For any distribution $\prob$ over trajectories, we define \emph{backward probabilities} as the conditional probabilities of the form $\prob(s_{h-1}, a_{h-1}\given s_h)$---note that conditioning is the opposite of transitions in $p$. For the backward probabilities to be defined, we do not need to fully specify a full distribution over trajectories, only a distribution $\nu$ over $(s_{h-1},a_{h-1})$. For any such distribution $\nu$, any $s\in\states_{h-1}$, $a\in\actions$ and $s'\in\states_h$, the backward probability is defined as
\begin{equation}
\label{eqn:bnu}
b_\nu(s,a\given s') = \frac{p(s'\given s,a)\,\nu(s,a) }{\sum_{\ts,\ta}
                            p(s'\given \ts,\ta)\,\nu(\ts,\ta)}
.
\end{equation}
For a given $s'\in \states_h$, we collect the probabilities $b_\nu(s,a\given s')$ across all $s\in\states_{h-1}$, $a\in\actions$ into the \emph{backward probability vector} $\vb_\nu(s')\in\simplex_{MK}$, padding with zeros if $\card{\states_{h-1}}<M$. Backward probability vectors are at the core of our approach, because they correspond to the state embeddings $\emb(s)$ approximated by our algorithms.
Our algorithms require that $\vb_{\nu}(s')$ for different states $s'\in\states_h$ be sufficiently separated
from one another for a suitable choice of $\nu$:
\begin{asmp}[$\gamma$-Separability]
	\label{asmp:identifiability}
	There exists $\gamma>0$ such that for any $h \in\set{2,\dotsc,H+1}$ and any distinct $s',s''\in\states_h$, the backward probability vectors with respect to the uniform distribution are separated by a margin of at least $\gamma$, i.e.,
$\norm{\vb_\nu(s')-\vb_\nu(s'')}_1\ge \gamma$, where $\nu=U(\states_{h-1}\times\actions)$.
\end{asmp}
We show in Section~\ref{sec:deterministic} that this assumption is automatically satisfied with $\gamma=2$ when latent-state transitions are deterministic (as assumed, e.g., by \citealp{dann2018polynomial}).
However, the class of $\gamma$-separable models is substantially larger.
In Appendix~\ref{sec:justification_pomdp} we show that the uniform distribution in the assumption can be replaced with any distribution supported on $\states_{h-1}\times\actions$, although the margins $\gamma$ would be different.

The key property that makes vectors $\vb_\nu(s')$ algorithmically useful is that they arise as solutions to a specific least squares problem with respect to data generated by a policy whose marginal distribution over $(s_{h-1},a_{h-1})$ matches $\nu$.
Let $\ve_{(s,a)}$ denote the vector of the standard basis in $\R^{MK}$ corresponding to the coordinate indexed by $(s,a)\in\states_{h-1}\times\actions$. Then the following statement holds:
%
%
%
\begin{thm}
\label{thm:why_pop_risk_minimizer}
Let $\nu$ be a distribution supported on $\states_{h-1}\times\actions$ and let $\tilde{\nu}$ be a distribution over $(s,a,x')$ defined by sampling $(s,a)\sim\nu$, $s'\sim p(\cdot\given s,a)$, and $x'\sim q(\cdot\given s')$. Let
\begin{align}
\label{eqn:pop_risk_minimization}
\vg_h\in\argmin_{\vg \in \gclass}
\expect_{\tilde{\nu}}\Bracks{\norm{\vg(x') - \ve_{(s,a)}}^2}.
\end{align}
Then, under Assumption~\ref{asmp:realizability}, every minimizer $\vg_h$ satisfies
$\vg_h(x')=\vb_\nu(s')$ for all $x'\in\contexts_{s'}$ and $s'\in\states_h$.
\end{thm}
The distribution $\tilde{\nu}$ is exactly the marginal distribution induced by a policy whose marginal distribution over $(s_{h-1},a_{h-1})$ matches $\nu$. Any minimizer $\vg_h$ yields context embeddings corresponding to state embeddings $\emb(s')=\vb_\nu(s')$. Our algorithms build on Theorem~\ref{thm:why_pop_risk_minimizer}: they replace the expectation by an empirical sample and obtain an approximate minimizer $\hvg_h$ by invoking an ERM oracle.

\section{Algorithm for Separable BMDPs}
\label{sec:pomdp}
\newcommand{\etah}{\eta_h}
\newcommand{\added}[1]{\color{red}#1\color{black}}

\begin{algorithm}[t]
	\caption{\textsc{PCID} (\textbf{P}olicy \textbf{C}over via \textbf{I}nductive \textbf{D}ecoding)} 
	\label{algo:stochastic}
	\begin{algorithmic}[1]
		\STATE \textbf{Input}:\\
        ~~~$\Nexp$: sample size for learning context embeddings\\
        ~~~$\Ncluster$: sample size for learning state embeddings\\
        ~~~$\Np$: sample size for estimating transition probabilities\\
        ~~~$\tau > 0$: a clustering threshold for learning latent states\\
\smallskip
		\STATE \textbf{Output}: policy cover $\Pi=\Pi_1\cup\cdots\cup\Pi_{H+1}$
\smallskip
        \STATE Let $\hatset_1=\set{s_1}$.
               Let $\hf_1(x)=s_1$ for all $x\in\contexts$.
        \STATE Let $\Pi_1=\set{\pi_0}$ where $\pi_0$ is the trivial $0$-step policy.
        \STATE Initialize $\hp$ to an empty mapping.
\smallskip
		\FOR{$h=2,\dotsc,H+1$}
		\STATE Let $\etah = \unif(\Pi_{h-1}) \odot \unif(\actions)$
\smallskip
        \STATE Execute $\etah$ for $\Nexp$ times.
               $\smash{\Dexp \!=\! \{\hats_{h-1}^i,a_{h-1}^i,x_h^i\}_{i=1}^{\Nexp}}$\\
        ~~~for $\smash{\hats_{h-1} \!=\! \hf_{h-1}(x_{h-1})}$.
\smallskip
		\STATE Learn $\hvg_h$ by calling ERM oracle on input $\Dexp$:\\
               ~~~$
                   \hvg_h = \argmin_{\vg \in \gclass} \sum_{(\hats,a,x')\in\Dexp} \norm{\vg(x')-\ve_{(\hats,a)}}^2$.
\vspace{-6pt}
        \STATE Execute $\etah$ for $\Ncluster$ times.
               $\cZ = \{\hvz_i = \hvg_h(x_h^i)\}_{i=1}^{\Ncluster}$.
\smallskip
		\STATE Learn $\smash{\hatset_h}$ and the state embedding map $\smash{\embh_h:\hatset_h\to\cZ}$\\
               ~~~by clustering $\cZ$ with threshold $\tau$ (see Algorithm~\ref{algo:find_representation}).
\vspace{-8pt}
		\STATE Define
               $\smash{\hatf_h(x') = \argmin_{\hat{s}\in\hatset_h}\,
                             \bigNorm{\embh(\hat{s}) - \hvg_h(x')}_1}$.
\medskip
		\STATE Execute $\etah$ for $\Np$ times.
               $\Dp \!=\! \{\hats^i_{h-1},a^i_{h-1},\hats^i_h\}_{i=1}^{\Np}$\\
        ~~~for $\smash{\hats_{h-1}\!=\!\hf_{h-1}(x_{h-1})}$, $\smash{\hats_h\!=\!\hf_h(x_h)}$.
\smallskip
		\STATE Define $\hp(\hats_h\given\hats_{h-1},a_{h-1})$\\
               ~~~equal to empirical conditional probabilities in $\Dp$.
\smallskip
		\FOR{$\hats' \in \widehat{\states}_h$}
		\STATE Run Algorithm~\ref{algo:dynamic} with inputs $\hp$ and $\hats'$\\
               ~~~to obtain $(h-1)$-step policy $\psi_{\hats'}:\hatset_{[h-1]}\to\actions$.
		\STATE Set $\pi_{\hat{s}'}(x_\ell) \!=\!\psi_{\hat{s}'}(\hat{f}_{\ell}(x_\ell))$,
               $\ell\in[h-1]$,  $x_\ell\in\contexts_\ell$.
		\ENDFOR
        \STATE Let $\Pi_h=\smash{(\pi_{\hats})_{\hats\in\hatset_h}}$.
		\ENDFOR
		\end{algorithmic}
\end{algorithm}

With the main components defined, we can now derive our algorithm for learning a policy cover in a separable BMDP.

The algorithm proceeds inductively, level by level. On each level $h$, we learn the following objects:
%
\begin{itemize*}
\item
The set of discovered latent states $\hatset_h\subseteq[M]$ and a decoding function $\smash{\hatf_h:\contexts \to \hatset_h}$, which allows us to identify latent states at level $h$ from observed contexts.

\item
The estimated transition probabilities $\hp(\hats_h\given \hats_{h-1},a)$ across all $\smash{\hats_{h-1}\in\hatset_{h-1}}$, $a\in\actions$, $\smash{\hats_h\in\hatset_h}$.

\item
A set of $(h-1)$-step policies $\Pi_h=\set{\pi_{\hats}}_{\hats\in\hatset_h}$.
\end{itemize*}
We establish a correspondence between the discovered states and true states via a bijection $\alpha_h$, under which the functions $\hf_h$ accurately decode contexts into states, the probability estimates $\hp$ are close to true probabilities, and $\Pi_h$ is an $\epsilon$--policy cover of $\states_h$.
%
%
%
Specifically, we prove the following statement for suitable accuracy parameters $\epsdec$, $\epsp$ and $\epsilon$:
%
%
\begin{claim}
\label{claim:stoc}
There exists a bijection $\alpha_h:\hatset_h\to\states_h$
such that the following conditions are satisfied for all $\hats\in\hatset_{h-1}$, $a\in\actions$, $\hats'\in\hatset_h$, and $s=\alpha_{h-1}(\hats_{h-1})$, $s'=\alpha_h(\hats')$, where $\alpha_{h-1}$ is the bijection for the previous level:
\begin{align}
&
  \text{Accuracy of $\hf_h$:}
&&
  \prob_{x'\sim q(\cdot\given s')}
  \bigBracks{\hf_h(x') =\hats'} \ge 1 - \epsdec,
\\
\notag
&
  \text{Accuracy of $\hp$:}
\\
  \span\span
\sum_{\hats''\in\hatset_h,\;s''=\alpha_h(\hats'')}\hspace{-2pt}
& \hspace{2pt}
  \BigAbs{\hp(\hats''\given\hats,a) - p(s''\given s,a)} \le \epsp,
\\
&
  \text{Coverage by $\Pi_h$:}\!\!
&&
  \prob^{\pi_{\hats'}}(s')\ge\mu(s')-\epsilon.
\end{align}
\end{claim}

Algorithm~\ref{algo:stochastic} constructs $\hatset_h$, $\hf_h$, $\hp$ and $\Pi_h$ level by level.
Given these objects up to level $h-1$, the construction for the next level~$h$ proceeds in the following three steps, annotated with the lines in Algorithm~\ref{algo:stochastic} where they appear:

\textbf{(1) Regression step: learn $\hvg_h$} (lines 7--9). We collect a dataset of trajectories by repeatedly executing a specific policy mixture $\etah$. We use $\hf_{h-1}$ to identify $\hats_{h-1}{=}\hf_{h-1}(x_{h-1})$ on each trajectory, obtaining samples $(\hats_{h-1},a_{h-1},x_h)$ from $\tnu$ induced by $\etah$.
The context embedding $\hvg_h$ is then obtained by solving the empirical version of~\eqref{eqn:pop_risk_minimization}.\looseness=-1

Our specific choice of $\etah$ ensures that each state $s_{h-1}$ is reached with probability at least $(\mu_{\min}-\epsilon)/M$, which is bounded away from zero if $\epsilon$ is sufficiently small. The uniform choice of actions then guarantees that each state on the next level is also reached with sufficiently large probability.

\textbf{(2) Clustering step: learn $\embh$ and $\hf_h$} (lines 10--12). Thanks to Theorem~\ref{thm:why_pop_risk_minimizer}, we expect that $\hvg_h(x')\approx\vg_h(x')=\vb_\nu(s')$ for the distribution $\nu(\hats_{h-1},a_{h-1})$ induced by $\etah$.\footnote{%
  Theorem~\ref{thm:why_pop_risk_minimizer} uses
  distributions $\nu$ and $\tnu$ over true states $s_{h-1}$, but its analog also holds for distributions over $\hats_{h-1}$, as long as decoding is approximately correct at the previous level.%
}
Thus, all contexts $x'$ generated by the same latent state $s'$ have embedding vectors $\hvg_h(x')$ close to each other and to $\vb_\nu(s')$. Thanks to separability,
we can therefore use clustering to identify all contexts generated by the same latent state, and this procedure is sample-efficient since the embeddings are low-dimensional vectors. Each cluster corresponds to some latent state $s'$ and any vector $\hvg_h(x')$ from that cluster can be used to define the state embedding $\smash{\embh(s')}$. The decoding function $\smash{\hf_h}$ is defined to map any context $x'$ to the state $s'$ whose embedding $\smash{\embh(s')}$ is the closest to $\hvg_h(x')$.

\textbf{(3) Dynamic programming: construct $\Pi_h$} (lines 13--19). Finally, with the ability to identify states at level $h$ via $\hf_h$, we can use collected trajectories to learn an approximate transition model $\hp(\hats'\given\hats,a)$ up to level $h$. This allows us to use dynamic programming to find policies that (approximately) optimize the probability of reaching any specific state $s'\in\states_h$. The dynamic programming finds policies $\psi_{\hats'}$ that act by directly observing decoded latent states. The policies $\pi_{\hats'}$ are obtained by composing $\psi_{\hats'}$ with the decoding functions $\smash{\{\hf_\ell\}_{\ell\in[h-1]}}$.\looseness=-1

\begin{algorithm}[t]
	\caption{Clustering to Find Latent-state Embeddings.}
	\label{algo:find_representation}
	\begin{algorithmic}[1]
		\STATE \textbf{Input}: Data points $\cZ=\set{\vz_i}_{i=1}^n$ and threshold $\tau > 0$.
		\STATE \textbf{Output}: Cluster indices $\hatset$ and centers $\embh:\hatset\to\cZ$.
\smallskip
		\STATE Let $\hatset = \emptyset$, $k = 0$ (number of clusters).
		\WHILE{$\cZ\neq \emptyset$}
		\STATE Pick any $\vz\in\cZ$ (a new cluster center).
		\STATE Let $\cZ'=\set{\vz'\in\cZ:\:\norm{\vz-\vz'}_1\le\tau}$.
        \STATE Add cluster: 
               $k\gets k+1$,~~$\hatset\gets\hatset\cup\set{k}$,~~$\embh(k)=\vz$.
        \STATE Remove the newly covered points: $\cZ\gets\cZ\setminus\cZ'$.
		\ENDWHILE
	\end{algorithmic}
\end{algorithm}

The next theorem guarantees that with a polynomial number of samples, Algorithm~\ref{algo:stochastic} finds a small $\epsilon$--policy cover.\footnote{%
The $\tO(\cdot)$, $\tOmega(\cdot)$, and $\tTheta(\cdot)$ notation suppresses factors that are polynomial in $\log M$, $\log K$, $\log H$ and $\log(1/\delta)$.}
\begin{thm}[Sample Complexity of Algorithm~\ref{algo:stochastic}]\label{thm:stochastic_pomdp_sample_complexity}
\renewcommand{\thefootnote}{\added{4.5}}
\added{Fix any $\delta>0$ and any $\epsilon = O\BigParens{\min\BigParens{\frac{\mu_{\min}^3\gamma}{M^4K^3H}, \frac{\delta\mu_{\min}}{MKH}}}$.}\footnote[1]{\added{The ICML 2019 version omitted the second constraint on $\epsilon$. We thank Yonathan Efroni for calling this to our attention.}}
\renewcommand{\thefootnote}{\arabig{footnote}}
Set $\Nexp=\tOmega\BigParens{\frac{M^4K^4H\log\,\card{\gclass}}{\epsilon\mu_{\min}^3\gamma^2}}$,
$\Ncluster=\tTheta\BigParens{\frac{MK}{\mu_{\min}}}$,
$\Np = \tOmega\BigParens{\frac{M^2KH^2}{\mu_{\min}\epsilon^2}}$,
$\tau = \frac{\gamma}{30MK}$.
Then with probability at least $1-\delta$,
Algorithm~\ref{algo:stochastic} returns an $\epsilon$--policy cover of $\states$, with size at most $MH$.
\end{thm}
%
In addition to dependence on the usual parameters like $M, K, H$ and $1/\epsilon$, our sample complexity also scales inversely with the separability margin $\gamma$ and the worst-case reaching probability $\mu_{\min}$. While the exact dependence on these parameters is potentially improvable, Appendix~\ref{sec:justification_pomdp} suggest that some inverse dependence is unavoidable for our approach. Compared with~\citet{azizzadenesheli2016pomdp}, there is no explicit dependence on $|\mathcal{X}|$, although they make spectral assumptions instead of the explicit block structure.

\begin{algorithm}[t]
	\caption{Dynamic Programming for Reaching a State}
	\label{algo:dynamic}
	\begin{algorithmic}[1]
		\STATE \textbf{Input}:~target state $\hats^* \in \hatset_h$,\\
               ~\hphantom{\textbf{Input}:}%
               transition probabilities $\hp(\hats'\given\hats,a)$\\
               ~\hphantom{\textbf{Input}:}~~~for all $\hats\in\hatset_\ell$, $a\in\actions$, $\hats'\in\hatset_{\ell+1}$, $\ell\in[h-1]$.
		\STATE \textbf{Output}: policy $\psi: \hatset_{[h-1]}\to\actions$ maximizing $\hat{\prob}^\psi(\hats^*)$.
\smallskip
		\STATE Let $v(\hats^*)=1$ and let $v(\hats)=0$ for all other $\smash{\hats\in\hatset_h}$.
		\FOR{$\ell = h-1,h-2,\ldots,1$}
		\FOR{$\hats \in \hatset_\ell$}
\vspace{-2pt}
		\STATE $\psi(\hats) = \max_{a \in \actions}\Bracks{\sum_{\hats'\in\hatset_{\ell+1}} v(\hats')\,\hp(\hats'\given \hats,a)}$.
        \STATE $v(\hats) = \sum_{\hats'\in\states_{\ell+1}} v(\hats')\,\hp(\hats'\given \hats,a=\psi(\hats))$.
		\ENDFOR
		\ENDFOR
	\end{algorithmic}
\end{algorithm}


\subsection{Deterministic BMDPs}
\label{sec:deterministic}

As a special case of general BMDPs, many prior works study the case of deterministic transitions, that is, $p(s' \given s,a) = 1$ for a unique state $s'$ for each $s,a$. Also, many simulation-based empirical RL benchmarks exhibit this property. We refer to these BMDPs as deterministic, but note that only the transitions $p$ are deterministic, not the emissions $q$. In this special case, the algorithm and guarantees of the previous section can be improved, and we present this specialization here, both for a direct comparison with prior work and potential usability in deterministic environments.

To start, note that $\mu_{\min}=1$ and $\gamma=2$ in any deterministic BDMP. The former holds as any reachable state is reached with probability one. For the latter, if $(s,a)$ transitions to $s'$, then $(s,a)$ cannot appear in the backward distribution of any other state $s''$. Consequently, the  backward probabilities for distinct states $s'\in\states_h$ must have disjoint support over $(s,a)\in\states_{h-1}\times\actions$, and thus their $\ell_1$ distance is exactly two.

Deterministic transitions allow us to obtain the policy cover with $\epsilon=0$; that is, we learn policies that are guaranteed to reach any given state $s$ with probability one. Moreover, it suffices to consider policies with simple structure: those that execute a fixed sequence of actions. Also, since we have access to policies reaching states in the prior level with probability one, there is no need for a decoding function $\hf_{h-1}$ when learning states and context embeddings on level $h$. The final, more technical implication of determinism (which we explain below) is that it allows us to boost the accuracy of the context embedding in the clustering step, leading to improved sample complexity.

\begin{algorithm}[tb]
	\caption{\textsc{PCID} for Deterministic BMDPs}
	\label{algo:deterministic}
	\begin{algorithmic}[1]
		\STATE \textbf{Input}:\\
        ~~~$\Nexp$: sample size for learning context embeddings\\
        ~~~$\Nboost$: sample size for boosting embedding accuracy\\
        ~~~$\tau > 0$: a clustering threshold for learning latent states\\
\smallskip
		\STATE \textbf{Output}: policy cover $\Pi=\Pi_1\cup\cdots\cup\Pi_{H+1}$
\smallskip
        \STATE Let $\hatset_1\!=\!\set{s_1}$.
               Let $\Pi_1\!=\!\set{\pi_0}$ for the $0$-step policy $\pi_0$.
\smallskip
		\FOR{$h=2,\dotsc,H+1$}
		\STATE Let $\eta_h = \unif(\Pi_{h-1}) \odot \unif(\actions)$
\smallskip
        \STATE Execute $\etah$ for $\Nexp$ times.
               $\smash{\Dexp \!=\! \{\hats_{h-1}^i,a_{h-1}^i,x_h^i\}_{i=1}^{\Nexp}}$\\
               ~~~where $\hats_{h-1}$ is the index of $\pi_{\hats_{h-1}}$ sampled by $\eta_h$.
\smallskip
		\STATE Learn $\hvg_h$ by calling the ERM oracle on input $\Dexp$:\\
               ~~~$
                   \hvg_h = \argmin_{\vg \in \gclass} \sum_{(\hats,a,x')\in\Dexp} \norm{\vg(x')-\ve_{(\hats,a)}}^2$.
\vspace{-8pt}
        \STATE Initialize $\cZ=\emptyset$ (dataset for learning latent states).
 		\FOR{$(\pi,a) \in \Pi_{h-1} \times \actions$}
        \STATE Execute $\pi\odot a$ for $\Nboost$ times. $\Dboost = \{x_h^i\}_{i=1}^{\Nboost}.$
		\STATE Set $\vz_{\pi \odot a} \!=\! \sum_{x\in\Dboost} \hvg_{h}(x)/\card{\Dboost}$,
               add $\vz_{\pi \odot a}$ to $\cZ$.
		\ENDFOR
		\STATE Learn $\hatset_h$ and the state embedding map $\smash{\embh_h:\hatset_h\to\cZ}$\\
               ~~~by clustering $\cZ$ with threshold $\tau$ (see Algorithm~\ref{algo:find_representation}).
\vspace{-10pt}
        \STATE
        Set $\Pi_h\!=\!\smash{(\pi_{\hats})_{\hats\in\hatset_t}}$
        where $\pi_{\hats}\!=\!\pi\odot a$ if $\embh_h(\hats)\!=\!\vz_{\pi\odot a}$.
		\ENDFOR
	\end{algorithmic}
\end{algorithm}

The details are presented in Algorithm~\ref{algo:deterministic}. At each level $h\in[H+1]$, we construct the following objects:
\begin{itemize*}
\item A set of discovered states $\hatset_h$.
\item A set of $(h-1)$-step policies $\Pi_h=\smash{\set{\pi_{\hats}}_{\hats\in\hatset_h}}$.
\end{itemize*}
We proceed inductively and for each level $h$ prove that the following claim holds with a high probability:
\begin{claim}
\label{claim:det}
There exists a bijection $\alpha_h:\hatset_h\to\states_h$ such that $\pi_{\hats}$ reaches $\alpha_h(\hats)$ with probability one.
\end{claim}
This implies that $\hatset_h$ can be viewed as a latent state space, and $\Pi_h$ is an $\epsilon$--policy cover of $\states_h$ with $\epsilon=0$.

To construct these objects for next level $h$, Algorithm~\ref{algo:deterministic} proceeds in three steps similar to Algorithm~\ref{algo:stochastic} for the stochastic case.  The regression step, that is, learning of $\hvg_h$ (lines 5--7), is identical. The clustering step (lines 8--13) is slightly more complicated. We boost the accuracy of the learned context embedding $\hvg_h$ by repeatedly sampling contexts that are guaranteed to be emitted from the same latent state (because they result from the same sequence of actions), and taking an average. This step allows us to get away with a lower accuracy of $\hvg_h$ compared with Algorithm~\ref{algo:stochastic}. Finally, the third step, learning of $\Pi_h$ (line 14), is substantially simpler.  Since any action sequence reaching a given cluster can be picked as a policy to reach the corresponding latent state, dynamic programming is not needed.

The following theorem characterizes the sample complexity of Algorithm~\ref{algo:deterministic}.
It shows we only need
$\tO\bigParens{M^2K^2H \log\,\card{\gclass}}$
samples to find a policy cover with $\epsilon=0$.
\begin{thm}[Sample Complexity of Algorithm~\ref{algo:deterministic}]
	\label{thm:sample_complexity_deterministic}
Set  $\tau = 0.01$,
$\Nexp = \tOmega(M^2K^2 \log\,\card{\gclass})$
and $\Nboost = \tOmega(MK)$.
Then with probability at least $1-\delta$, Algorithm~\ref{algo:deterministic}
returns an $\epsilon$--policy cover of $\states$, with $\epsilon=0$ and size at most $MH$.
\end{thm}

In Appendix~\ref{app:rewards}, we discuss how to use policy cover to optimize a reward. For instance, if the reward depends on the latent state, the policy cover enables us to reach each state-action pair
and collect $O(1/\epsilon^2)$ samples to estimate this pair's expected reward up to $\epsilon$ accuracy.
Thus, using $O(MKH/\epsilon^2)$ samples in addition to those needed by Algorithm~\ref{algo:deterministic}, we can find the trajectory with the largest expected reward within an $H\epsilon$ error. To summarize:
\begin{corollary}
\label{cor:pac_deterministic}
With probability at least $1-\delta$, Algorithm~\ref{algo:deterministic} can be used to find an $\epsilon$-suboptimal policy using at most
$\tO\bigParens{M^2K^2H\log\,\card{\gclass} + MKH^3/\epsilon^2}$
trajectories from a deterministic BMDP.
\end{corollary}
This corollary (proved in Appendix~\ref{sec:proofs_det} as Corollary~\ref{cor:pac_deterministic_app}) significantly improves over the prior bound $O(M^3H^8K/\epsilon^5)$
obtained by~\citet{dann2018polynomial}, although their function-class complexity term is not directly comparable to ours, as their work approximates optimal value functions and policies, while we approximate ideal decoding functions.



\section{Experiments}
\label{sec:exp}
\begin{figure*}
\begin{center}
\includegraphics[width=\textwidth]{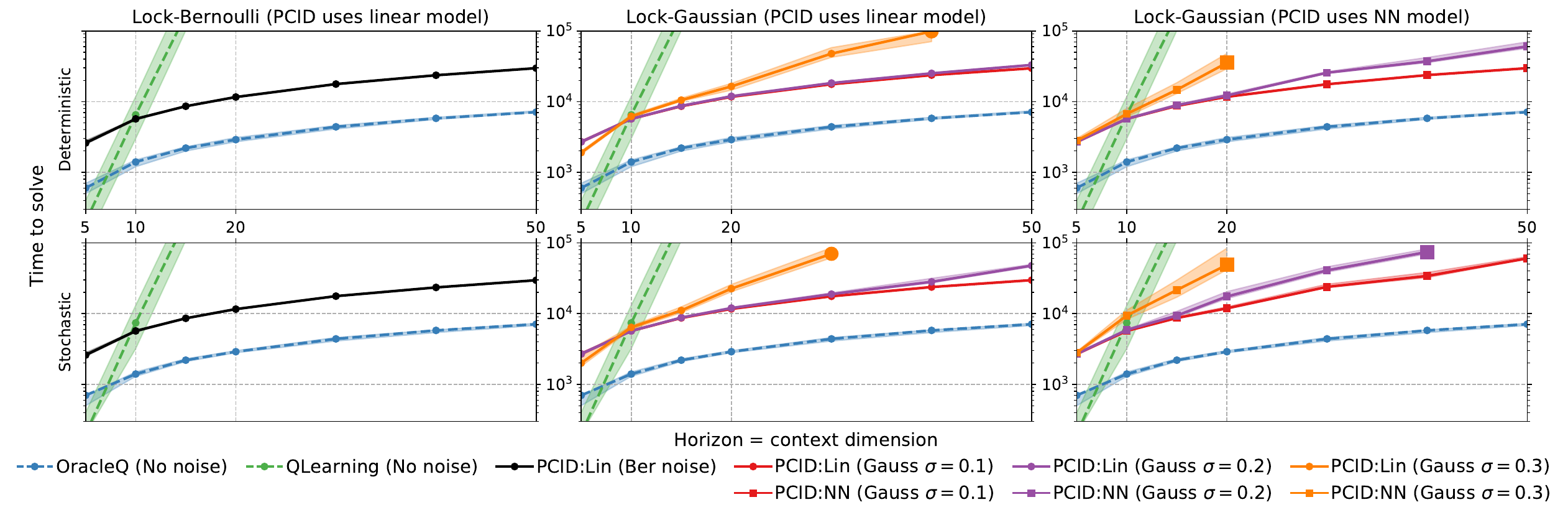}
\end{center}
\vspace{-13pt}%
\footnotesize{\emph{Note: Larger markers mean that the next point is off the plot.}}%
\vspace{-4pt}
\caption{Time-to-solve against problem difficulty for the combination lock
  environment with two observation processes and two function approximation
  classes. Left: \emph{Lock-Bernoulli}
  with linear functions. Center: \emph{Lock-Gaussian} with linear functions.
  Right: \emph{Lock-Gaussian} with neural networks. Top row:
  deterministic latent transitions. Bottom row: stochastic transitions
  with switching probability 0.1.
\oracleq and \qlearning are cheating and operate directly
  on latent states.}
\label{fig:lock_results}
\end{figure*}


We perform an empirical evaluation of our decoding-based algorithms in
six challenging RL environments, with two choices of the function class
$\mathcal{G}$. We compare our algorithm, which operates directly on
rich observations, against two tabular algorithms,
which operate on the latent state: a sanity-check baseline and a near-optimal skyline.
Some of the environments meet the BMDP assumptions and some do not; the former validate our theoretical results, while the latter demonstrate our algorithm's robustness.
Our code is available at \textit{https://github.com/Microsoft/StateDecoding}.

\textbf{The environments.}  All environments share the same latent
structure, and are a form of a ``combination lock," with $H$ levels, 3
states per level, and 4 actions.
Non-zero reward is only achievable from states $s_{1,h}$ and $s_{2,h}$.
From $s_{1,h}$ and $s_{2,h}$ one action leads with probability $1-\alpha$
to $s_{1,h+1}$ and with probability $\alpha$ to $s_{2,h+1}$, another has the flipped behavior, and the
remaining two lead to $s_{3,h+1}$. All actions from $s_{3,h}$
lead to $s_{3,h+1}$. The ``good" actions are randomly assigned
for every state. From $s_{1,H}$ and $s_{2,H}$, two actions
receive $\textrm{Ber}(1/2)$ reward; all others provide zero
reward. The start state is $s_{1,1}$. We consider deterministic variant ($\alpha=0$)
and stochastic variant ($\alpha=0.1$). (See
Appendix~\ref{sec:exp_app}.)


The environments are designed to be difficult for exploration. For
example, the deterministic variant has $2^H$ paths with non-zero
reward, but $4^H$ paths in total, so random exploration requires
exponentially many trajectories.

We also consider two observation processes, which we use only for our
algorithm, while the baseline and the skyline operate directly on the latent state space.
In \emph{Lock-Bernoulli}, the observation space is
$\{0,1\}^{H+3}$ where the first $3$ coordinates are reserved for
one-hot encoding of the state and the last $H$ coordinates are drawn
i.i.d. from $\textrm{Ber}(1/2)$. The observation space
is not partitioned across time, which our algorithms track
internally. Thus, \emph{Lock-Bernoulli} meets the BMDP assumptions
and can be perfectly decoded via linear functions.  In \emph{Lock-Gaussian}, the observation space is
$\mathbb{R}^{H+3}$. As before the first $3$ coordinates are reserved
for one-hot encoding of the state, but this encoding is corrupted with
Gaussian noise. Formally, if the agent is at state $s_{i,h}$ the
observation is $\ve_i +\vv \in \mathbb{R}^{3+H}\!$, where $\ve_i$ is one of
the first three standard basis vectors and $\vv$ has
$\mathcal{N}(0,\sigma^2)$ entries.  We consider $\sigma \in
\{0.1,0.2,0.3\}$. Note that Lock-Gaussian \emph{does not} satisfy
Assumption~\ref{asmp:block} since the emission distributions
cannot be perfectly separated.
We use this environment to evaluate the robustness of our algorithm to violated assumptions.%
\looseness=-1


\textbf{Baseline, skyline, hyperparameters.}
We compare our algorithm against two \emph{tabular} approaches that
cheat by directly accessing the latent state. The first, \oracleq, is the
Optimistic $Q$-Learning algorithm of~\citet{jin2018q},
with a near-optimal regret in tabular environments.\footnote{We use the Hoeffding version, which
  is conceptually much simpler, but statistically slightly worse.}
Because of its near-optimality and direct access to the latent state, we do not expect any algorithm to beat \oracleq,
and view it as a skyline.
The second, \qlearning, is tabular $Q$-learning with $\epsilon$-greedy
exploration.
It serves as a sanity-check baseline:
any algorithm with strategic exploration should vastly outperform
\qlearning, even though it is cheating.


Each algorithm has two hyperparameters that we tune.
In our algorithm (\decoding), we use $k$-means clustering
instead of Algorithm~\ref{algo:find_representation},
so one of the hyperparameters is the number of clusters $k$. The second
one is the number of trajectories $n$ to collect in each outer iteration.
For \oracleq, these are the
learning rate $\alpha$ and a confidence parameter $c$. For \qlearning,
these are the learning rate $\alpha$ and $\epsilon_{\textrm{frac}} \in
[0,1]$, a fraction of the 100K episodes over which to anneal the
exploration probability linearly from 1 down to 0.01.

For both \emph{Lock-Bernoulli} and \emph{Lock-Gaussian}, we experiment
with linear decoding functions, which we fit via ordinary least
squares. For \emph{Lock-Gaussian} only, we also use two-layer neural
networks. Specifically, these functions are of the form $f(\vx) =
\mW_2^\top \textrm{sigmoid}(\mW_1^\top\vx +\vc)$ with the standard sigmoid
activation, where the inner dimension is set to the clustering
hyper-parameter $k$. These networks are trained using AdaGrad with a
fixed learning rate of $0.1$, for a maximum of 5K iterations. See
Appendix~\ref{sec:exp_app} for more details on hyperparameters and
training.

\textbf{Experimental setup.}
We run the algorithms on all environments with varying
$H$, which also influences the dimension of the observation
space. Each algorithm runs for 100K episodes and we say that it
has \emph{solved the lock} by episode $t$ if at round $t$ its running-average
reward is $\geq 0.25 = 0.5V^\star$. The \emph{time-to-solve} is the
smallest $t$ for which the algorithm has solved the lock. For each
hyperparameter, we run 25 replicates with different randomizations of
the environment and seeds, and we plot the median time-to-solve of the
best hyperparameter setting (along with error bands corresponding to
$90^{\textrm{th}}$ and $10^{\textrm{th}}$ percentiles) against the
horizon $H$.
Our algorithm is reasonably fast, e.g., a single replicate
of the above protocol for the two-layer neural net model and $H=50$ takes less than 10 minutes
on a standard laptop.

\textbf{Results.}
The results are in Figure~\ref{fig:lock_results} in a log-linear
plot. First, \qlearning works well for small horizon problems but
cannot solve problems with $H \geq 15$ within 100K episodes, which is
not surprising.\footnote{We actually ran \qlearning for 1M episodes
and found it solves $H=15$ with 170K episodes.}  The performance curve
for \qlearning is linear, revealing an exponential sample complexity,
and demonstrating that these environments cannot be solved with
na\"{i}ve exploration. As a second observation, \oracleq performs
extremely well, and as we verify in Appendix~\ref{sec:exp_app}
demonstrates a linear scaling with $H$.\footnote{This is incomparable
with the result in~\citet{jin2018q} since we are not measuring regret
here.}


In \emph{Lock-Bernoulli}, \decoding is roughly a factor of 5 worse
than the skyline \oracleq for all values of $H$, but the curves have
similar behavior.  In Appendix~\ref{sec:exp_app}, we verify a
\emph{near-linear} scaling with $H$, even better than predicted by
our theory.  Of course \decoding is an exponential improvement
over \qlearning with $\epsilon$-greedy exploration here.

In \emph{Lock-Gaussian} with linear functions, the results are similar
for the low-noise setting. The performance of \decoding degrades
as the noise level increases. For example, with noise level $\sigma =
0.3$, it fails to solve the stochastic problem with $H=40$ in 100K
episodes. This is expected, as Assumption~\ref{asmp:block} is severely violated at this noise level. However, the scaling of
the sampling complexity still represents a dramatic improvement over \qlearning.

Finally, \decoding with neural networks is less robust to noise and
stochasticity in \emph{Lock-Gaussian}. Here, with $\sigma=0.3$ the
algorithm is unable to solve the $H=30$ problem, both with and without
stochasticity, but still does quite well with $\sigma \in
\{0.1,0.2\}$. The scaling with $H$ is still quite favorable.


\textbf{Sensitivity analysis.}
  We also perform a simple sensitivity analysis to assess how
  the hyperparameters $k$ and $n$ influence the behavior of \decoding.
  We find that if we
  under-estimate either $k$ or $n$ the algorithm fails, either because
  it cannot identify all latent states, or it does not collect enough
  data to solve the regression problems. On the other hand,
  the algorithm is quite robust to over-estimating both parameters.
  (See Appendix~\ref{app:sensitivity} for further details.)


\textbf{Summary.}
We have shown on several rich-observation environments with both
linear and non-linear functions that \decoding scales to large-horizon
rich-observation problems. It dramatically outperforms
tabular \qlearning with $\epsilon$-greedy exploration, and is roughly
a factor of 5 worse than a near-optimal \oracleq with an access
to the latent state.
\decoding's performance is robust to hyperparameter
choices and degrades gracefully as the assumptions are
violated.

%
%

\bibliography{simonduref}
\bibliographystyle{icml2019}

\onecolumn
\newpage
\appendix

\section{Comparison of BMDPs with other related frameworks}
\label{sec:comparison}
The problem setup in a BMDP is closely related to the literature on state abstractions, as our decoding function can be viewed as an abstraction over the rich context space. Since we learn the decoding function instead of assuming it given, it is worth comparing to the literature on state abstraction learning.  The most popular notion of abstraction in model-based RL is bisimulation \cite{whitt1978approximations, givan2003equivalence}, which is more general than our setup since
our context is sampled i.i.d.~conditioned on the hidden state (the irrelevant factor discarded by a bisimulation may not be i.i.d.). Such generality comes with a cost as learning good abstractions turns out to be very challenging.   The very few results that come with finite sample guarantees can only handle a small number of candidate abstractions \citep{ hallak2013model, ortner2014selecting, pmlr-v37-jiang15}.
In contrast, we are able to learn a good decoding function from an exponentially large and unstructured family (that is, the decoding functions $g_h \in \gclass$ combined with the state encodings $\phi$). 

The setup and algorithmic ideas in our paper are related to the work 
of~\citet{azizzadenesheli2016pomdp, azizzadenesheli2016romdp}, but we are able to handle continuous observation spaces with no direct dependence on the number of unique contexts due to the use of function approximation. The recent setup of Contextual Decision Processes (CDPs) with low Bellman rank, introduced by~\citet{jiang2017contextual} is a strict generalization of BMDPs (the Bellman rank of any BMDP is at most $M$). The additional assumptions made in our work enable the development of a computationally efficient algorithm, unlike in their general setup. Most similar to our work,~\citet{dann2018polynomial} study a subclass of CDPs with low Bellman rank where the transition dynamics are deterministic.\footnote{While not explicitly assumed in their work, the assumption of the optimal policy and value functions depending only on the current observation and not hidden state is most reasonable when the observations are disjoint across hidden states like in this work.} However, instead of the deterministic dynamics in \citeauthor{dann2018polynomial}, we consider stochastic dynamics with certain reachability and separability conditions. As we note in Section~\ref{sec:pomdp}, these assumptions are trivially valid under deterministic transitions. In terms of the realizability assumptions, Assumption~\ref{asmp:realizability} posits the realizability of a decoding function, while~\citet{dann2018polynomial} assume realizability of the optimal value function. These assumptions are not directly comparable, but are both reasonable if the decoding and value functions implicitly first map the contexts to hidden states, followed by a tabular function as discussed after Assumption~\ref{asmp:realizability}. Finally as noted by \citeauthor{dann2018polynomial}, certain empirical RL benchmarks such as visual grid world are captured reasonably well in our setting.

On the empirical side, \cite{pathak2017curiosity} learn a encoding function that compresses the rich obervations to a low-dimensional representation, which serves a similar purpose as our decoding function,
using prediction errors in the low-dimensional space to drive exploration.  This approach has weaknesses, as it cannot cope with stochastic transition structures.  Given this, our work can also be viewed as a rigorous fix for these types of empirical heuristics.

\section{Incorporating Rewards in BMDPs}
\label{app:rewards}
At a high level, there are two natural choices for modeling rewards in a BMDP. In some cases, the rewards might only depend on the latent state. This is analogous to how rewards are typically modeled in the POMDP literature and respects the semantics that $s$ is indeed a valid state to describe an optimal policy or value function. For such problems, finding a near optimal policy or value function building on Algorithms~\ref{algo:stochastic} or~\ref{algo:deterministic} is relatively straightforward. Note that along with the policy cover, our algorithms implicitly construct an approximately correct dynamics model $\hp$ in the latent state space as well as decoding functions $\hf$ which map contexts to the latent states generating them with a small error probability. While these objects are explicit in Algorithm~\ref{algo:stochastic}, they are implicit in Algorithm~\ref{algo:deterministic} since each policy in the cover reaches a unique latent state with probability 1 so that we do not need any decoding function. Indeed for deterministic BMDPs, we do not need the dynamics model at all given the policy cover to maximize a state-dependent reward as shown in Corollary~\ref{cor:pac_deterministic}. For stochastic BMDPs, given any reward function, we can simply plan within the dynamics model over the latent states to obtain a near-optimal policy as a function of the latent state. We construct a policy $\widehat{\pi}$ as a function of contexts by first decoding the context using $\hf$ and then applying the near-optimal policy over latent states found above. As we show in the main text, there are parameters $\epsdec$ and $\epsp$ controlled by our algorithms, such that the policy found using the procedure described above is at most $O\bigParens{H(\epsdec + \epsp)}$ suboptimal.

In the second scenario where the reward depends on contexts, the optimal policies and value functions cannot be constructed using the latent states alone. However, our policy cover can still be used to generate a good exploration dataset for subsequent use in off-policy RL algorithms, as it guarantees good coverage for each state-action pair. Concretely, if we use value-function approximation, then the dataset can be fed into an approximate dynamic programming (ADP) algorithm~\citep[e.g., FQI][]{ernst2005tree}. Given a good exploration dataset, these approaches succeed under certain representational assumptions on the value-function class \citep{antos2008learning}. Similarly, one can use PSDP style policy learning methods on such a dataset~\citep{bagnell2004policy}.

We conclude this subsection by observing that in reward maximization for RL, most works fall into either seeking a PAC or a regret guarantee. Our approach of first constructing a policy cover and then learning policies or value functions naturally aligns with the PAC criterion, but not with regret minimization. Nevertheless, as we see in our empirical evaluation, for challenging RL benchmarks, our approach still has a good performance in terms of regret.

\section{Experimental Details and Reproducibility Checklist}
\label{sec:exp_app}
\subsection{Implementation Details}

\paragraph{Environment transition diagram.}
The hidden state transition diagram for the Lock environment is
displayed in Figure~\ref{fig:lock_structure}.
\begin{figure}
\begin{center}
\begin{tikzpicture}
\draw[black] (0,0) circle (12pt) node (s11) {$s_{1,1}$};
\draw[black] (2.5,0) circle (12pt) node (s12) {$s_{1,2}$};
\draw[black] (5.0,0) circle (12pt) node (s13) {$s_{1,3}$};
\coordinate[] (d11a) at (1.0,0.25) {}; 
\coordinate[] (d11b) at (0.75,-0.25) {}; 
\coordinate[] (d12a) at (3.5,0.25) {}; 
\coordinate[] (d12b) at (3.25,-0.25) {}; 
\node[] (d13a) at (7.0, 0.25) {$B(\tfrac{1}{2})$};
\node[] (d13b) at (7.0, -0.25) {$0$};

\draw[black] (0,-1.5) circle (12pt) node (s21) {$s_{2,1}$};
\draw[black] (2.5,-1.5) circle (12pt) node (s22) {$s_{2,2}$};
\draw[black] (5.0,-1.5) circle (12pt) node (s23) {$s_{2,3}$};
\coordinate[] (d21a) at (0.75,-1.25) {}; 
\coordinate[] (d21b) at (1.0,-1.75) {}; 
\coordinate[] (d22a) at (3.25,-1.25) {}; 
\coordinate[] (d22b) at (3.5,-1.75) {}; 
\node[] (d23a) at (7.0, -1.25) {$B(\tfrac{1}{2})$};
\node[] (d23b) at (7.0, -1.75) {$0$};

\draw[black] (0,-3.0) circle (12pt) node (s31) {$s_{3,1}$};
\draw[black] (2.5,-3.0) circle (12pt) node (s32) {$s_{3,2}$};
\draw[black] (5.0,-3.0) circle (12pt) node (s33) {$s_{3,3}$};
\node[] (d33) at (7.0, -3.0) {$0$};

\path[black,thick,->] (s11) edge node {} (d11a);
\path[black, thick, ->] (s11) edge node {} (d11b);
\path[green, thick, bend left, ->] (d11a) edge node {} (s12);
\path[blue, thick, bend right, ->] (d11a) edge node {} (s22);
\path[green, thick, bend right, ->] (d11b) edge node {} (s22);
\path[blue, thick, bend left, ->] (d11b) edge node {} (s12);
\path[black, thick, ->] (s11) edge node {} (s32);

\path[black,thick,->] (s12) edge node {} (d12a);
\path[black, thick, ->] (s12) edge node {} (d12b);
\path[green, thick, bend left, ->] (d12a) edge node {} (s13);
\path[blue, thick, bend right, ->] (d12a) edge node {} (s23);
\path[green, thick, bend right, ->] (d12b) edge node {} (s23);
\path[blue, thick, bend left, ->] (d12b) edge node {} (s13);
\path[black, thick, ->] (s12) edge node {} (s33);

\path[black, thick, bend left, ->] (s13) edge node {} (d13a);
\path[black, thick, ->] (s13) edge node {} (d13a);
\path[black, thick, ->] (s13) edge node {} (d13b);
\path[black, thick, bend right, ->] (s13) edge node {} (d13b);

\path[black, thick, ->] (s21) edge node {} (d21a);
\path[black, thick, ->] (s21) edge node {} (d21b);
\path[green, thick, bend left, ->] (d21a) edge node {} (s12);
\path[blue, thick, bend right, ->] (d21a) edge node {} (s22);
\path[green, thick, bend right, ->] (d21b) edge node {} (s22);
\path[blue, thick, bend left, ->] (d21b) edge node {} (s12);
\path[black, thick, ->] (s21) edge node {} (s32);

\path[black, thick, ->] (s22) edge node {} (d22a);
\path[black, thick, ->] (s22) edge node {} (d22b);
\path[green, thick, bend left, ->] (d22a) edge node {} (s13);
\path[blue, thick, bend right, ->] (d22a) edge node {} (s23);
\path[green, thick, bend right, ->] (d22b) edge node {} (s23);
\path[blue, thick, bend left, ->] (d22b) edge node {} (s13);
\path[black, thick, ->] (s22) edge node {} (s33);

\path[black, thick, bend left, ->] (s23) edge node {} (d23a);
\path[black, thick, ->] (s23) edge node {} (d23a);
\path[black, thick, ->] (s23) edge node {} (d23b);
\path[black, thick, bend right, ->] (s23) edge node {} (d23b);

\path[black, thick, ->] (s31) edge node {} (s32);
\path[black, thick, ->] (s32) edge node {} (s33);
\path[black, thick, ->] (s33) edge node {} (d33);
\end{tikzpicture}
\end{center}
\caption{Lock transition diagram. The process is layered with time
  moving from left to right. Green arrows denote high probability
  transitions (either $1.0$, or $0.9$) while blue arrows denote low
  probability transitions ($0.0$ or $0.1$). The agent starts in
  $s_{1,1}$. All states have four actions (all actions have the same
  effect for $s_{3,h}$), and the action labels are randomized for each
  replicate.}
\label{fig:lock_structure}
\end{figure}
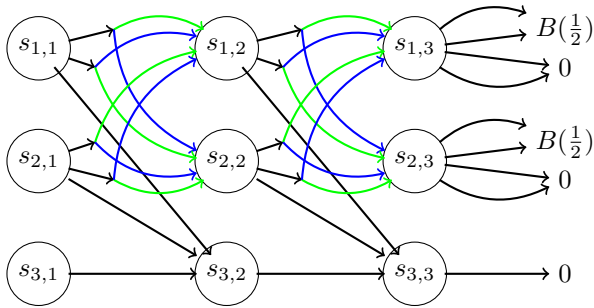

Our implementation of \decoding follows
Algorithm~\ref{algo:stochastic}, with a few small differences. First,
we set $N_{g}=N_{\phi}=N_p=n$, where $n$ is a tuned
hyperparameter. The first data collection step in Line 8 is as
described: uniform over $\Pi_{h-1}\circ \mathcal{A}$ for $n$
samples. The oracle in Line 9 is implemented differently for each
representation as we detail below. Then rather than collect $n$
additional samples in Line 10, we simply re-use the data from Line
8. For clustering, as mentioned, we use $K$-means clustering rather
than the canopy-style subroutine described in
Algorithm~\ref{algo:find_representation}. We describe this in more
detail below. We re-use data in lieu of the last data-collection step
in Line 13, and the transition probabilities are estimated simpy via
empirical frequencies.  Finally, the policies $\Pi_h$ are learned via
dynamic programming on the learned latent transition model.

The two steps that require further clarification are the
implementation of the oracle and the clustering step.

\paragraph{Oracle Implementation and representation.}
Whenever we use \decoding with a linear representation, we use
unregularized linear regression, e.g., ordinary least squares. Since
we have vector-valued predictions, we performan linear regression
independently on each coordinate. Formally, with data matrix $X \in
\mathbb{R}^{n \times d}$ and targets $Y \in \mathbb{R}^{n \times p}$ the parameter
matrix $\hat{\beta} \in \mathbb{R}^{d \times p}$ is
\begin{align*}
\hat{\beta} = (X^\top X)^{-1} X^\top Y
\end{align*}
We solve for $\hat{\beta}$ exactly, modulo standard numerical methods
for performing the matrix inverse (specifically,
\texttt{numpy.linalg.pinv}). Note that we do not add an intercept term
to this problem.

When we use a neural network oracle, the representation is always
$f(x) = W_2^\top \textrm{sigmoid}(W_1^\top x + c)$. For dimensions, if
the clustering hyper-parameter is $k$, the observation space has
dimension $d$, and the targets for the regression problem have
dimension $p$, then the weight matrices have $W_2 \in \mathbb{R}^{k
  \times p}, W_1 \in \mathbb{R}^{d \times k}$ and the intercept term
is $c \in \mathbb{R}^k$. $\textrm{sigmoid}(z) = (1+e^{-z})^{-1}$ is
the standard sigmoid activation, and we always use the square loss.
To fit the model, we use AdaGrad with a fixed learning rate multiplier
of $0.1$. With output dimension $p$, each iteration of optimization
makes $p$ updates to the model, one for each output dimension in
ascending order. As a rudimentary convergence test, we compute the
total training loss (over all output dimensions) at each
iteration. For each $t \in \mathbb{N}$, we check if the training loss
at iteration $100 t$ is within $10^{-3}$ of the training loss at round
$100(t-1)$. If so, we terminate optimization. We always terminate
after 5000 iterations. Our neural network model and training is
implemented in \texttt{pytorch}.

\paragraph{Clustering.}
We use the \texttt{scikit-learn} implementation of $K$-means
clustering for clustering in the latent embedding space, with a simple
model selection subroutine as a wrapper to tune the number of
clusters. Starting with $k$ set to the hyperparameter used as input to
\decoding, we run $K$-means, searching for $k$ clusters, and we check
if each found clusters has at least $30$ points. If not, we decrease
$k$ and repeat.

\subsection{Additional Results}

\begin{figure*}
\includegraphics[width=\textwidth]{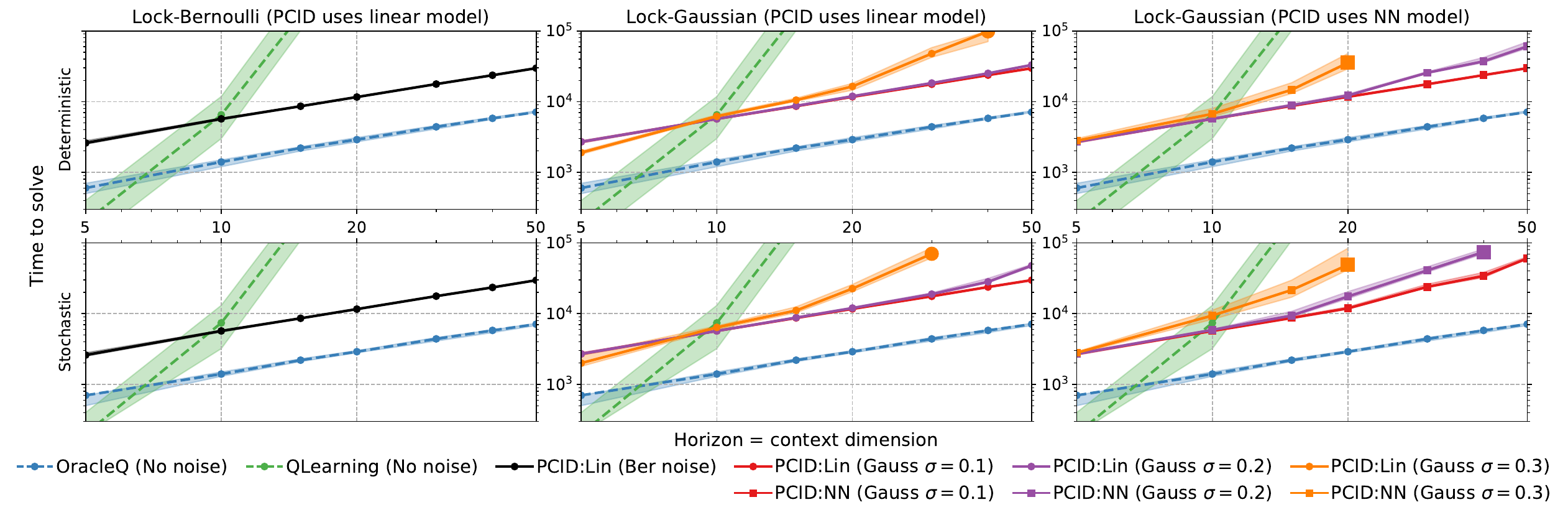}
\caption{Time-to-solve against problem difficulty for the
  \texttt{Lock} environment with two different observation processes
  and function classes, plotted now in a log-log plot. The curves
  confirm a linear scaling with difficult for both \decoding and
  \oracleq.}
\label{fig:lock_solve_loglog}
\end{figure*}

\begin{figure*}
\includegraphics[width=\textwidth]{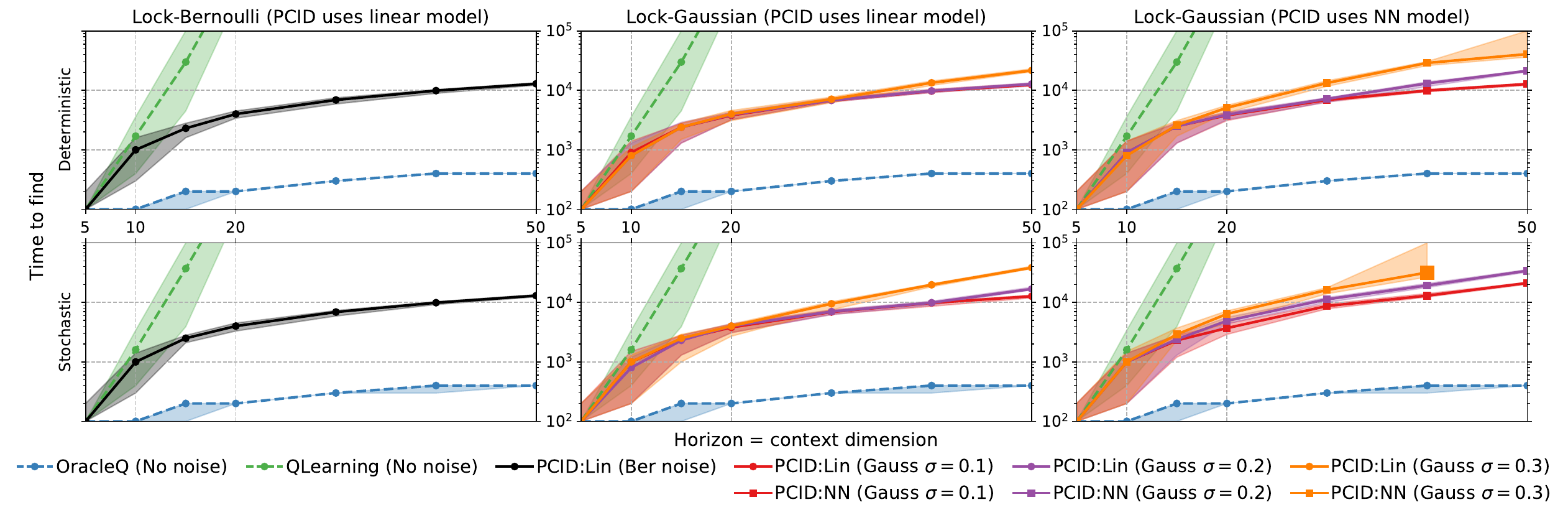}
\caption{Time-to-find the goal against problem difficulty for the
  \texttt{Lock} environment with two different observation processes
  and function classes. Left: Lock-Bernoulli environment. Center:
  Lock-Gaussian with linear functions, Right: Lock-Gaussian with
  neural networks. Top row: deterministic latent transitions. Bottom
  row: stochastic transitions with switching probability 0.1.
  \oracleq and \qlearning operate directly on hidden states, and hence
  are invariant to observation processes.}
\label{fig:lock_find}
\end{figure*}

In Figure~\ref{fig:lock_solve_loglog} we plot exactly the same results
as in Figure~\ref{fig:lock_results} except we visualize the results on
a log-log plot. This verifies the linear scaling with $H$ for both
\oracleq and \decoding in \emph{Lock-Bernoulli} and
\emph{Lock-Gaussian} with linear functions. The slope for the
line-of-best fit for \oracleq in the deterministic setting is $1.065$
and in the stochastic setting it is $1.013$. For \oracleq, this
corresponds to the exponent on $H$ in the sample complexity. On
\emph{Lock-Bernoulli}, \decoding has slope $1.051$ in both settings,
in the log-log scale, as above this corresponds to the exponent on $H$
in our sample complexity, but since we have bound $d=H$ in these
experiments, a linear dependence on $H$ is substantially better than
what our theory predicts.

In Figure~\ref{fig:lock_find} we use a different performance measure
to compare the three algorithms, but all other details are identical
to the results in Figure~\ref{fig:lock_results}. Here we measure the
\emph{time-to-find}, which is the first episode for which the agent
has non-zero total reward. Since the environments have no immediate
reward, and almost all trajectories receive zero reward, this metric
more closely corresponds to solving the exploration problem, while
\emph{time-to-solve} requires exploration and exploitation. We use
time-to-solve in the main text because it is a better fit for the
baseline algorithms.

As before, we plot the median time-to-find with error bars
corresponding to $90^{\textrm{th}}$ and $10^{\textrm{th}}$ percentiles
for the best hyperparameter, over 25 replicates, for each algorithm
and in each environment. As sanity checks, \oracleq always finds the
goal extremely quickly, and \qlearning always fails for $H=20$, which
is unsurprising. Qualitatively the results for \decoding are similar
to those in Figure~\ref{fig:lock_results}, but notice that with neural
network representation, \decoding almost always finds the goal in 100K
episodes, even if it is unable to accumulate high reward. This
suggests either a failure in \emph{exploitation}, which is not the
focus of this work, or that the agent would solve the problem with a
few more episodes.

\begin{figure}
\begin{center}
\includegraphics[width=0.4\textwidth]{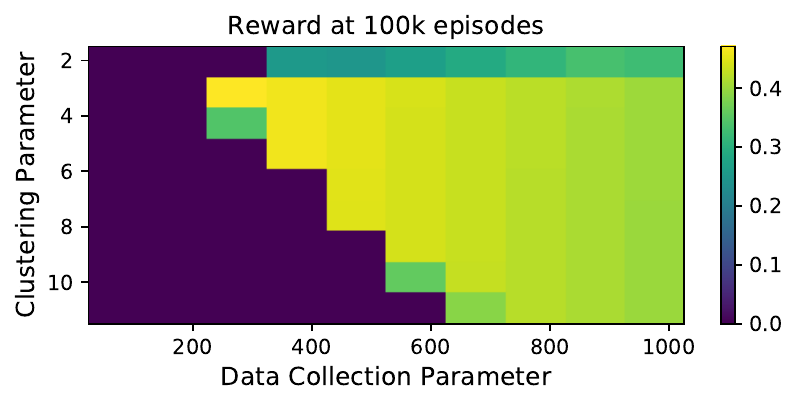}
\end{center}
\vspace{-0.5cm}
\caption{Sensitivity analysis for \decoding on
  \emph{Lock-Bernoulli} with $H=20$, showing robustness to overestimating hyperparameters. }
\label{fig:sensitivity}
\end{figure}

\subsection{Sensitivity Analysis}
\label{app:sensitivity}

We perform a simple sensitivity analysis to assess how the
hyperparameters $k$ and $n$ influence the behavior of \decoding. In
Figure~\ref{fig:sensitivity} we display a heat-map showing
the running-average reward (taking median over 25 replicates) of the
algorithm on the stochastic \emph{Lock-Bernoulli} environment with
$H=20$ as we vary both $n$ and $k$. The best parameter choice here is
$k=3$ and $n=300$. As we expect, if we under-estimate either $k$ or
$n$ the algorithm fails, either because it cannot identify all latent
states, or it does not collect enough data to solve the induced
regression problems. On the other hand, the algorithm is quite robust
to over-estimating both parameters, with a graceful degradation in
performance.

\subsection{Reproducibility Checklist}
\begin{itemize}
\item \textbf{Data collection process.} There was no dataset
  collection for this paper, but see below for details about
  environments and how results were collected.
\item \textbf{Datasets/Environments.} Environments are implemented in
  the OpenAI Gym API. Source code for environments are included with
  submission and will be made publicly available.
\item \textbf{Train/Validation/Test Split.} Our performance metrics
  are akin to regret, and require no train/test split. Nevertheless,
  we used different random seeds for development and for the final
  experiment.
\item \textbf{Excluded data.} No data was excluded.
\item \textbf{Hyperparmeters.} For \oracleq and \qlearning we consider
  learning rates in $\{1^{-x} : x \in \{-4,\ldots,0\}\}$. For \oracleq
  we chooose the confidence parameter from the same set. For
  \qlearning the exploration parameter, the fraction of the learning
  process over which to linearly decay the exploration probability, is
  chosen from $\{0.0001,0.001,0.01,0.1,0.5\}$. For \decoding, in
  Figures~\ref{fig:lock_results} and~\ref{fig:lock_find}, we set the
  K-means parameter to 3 and we choose $n \in
  \{100,200,\ldots,900,1000\}$. Hyperparameters were selected as
  follows: for each environment/horizon pair we choose the
  hyperparameter with best $90^{\textrm{th}}$ percentile
  performance. If the $90^{\textrm{th}}$ percentile for all
  hyperparameters exceeds the training time, we choose the
  hyperparameter with the best median performance. If this fails we
  optimize for $10^{\textrm{th}}$ percentile performance.
\item \textbf{Evaluation Runs.} We always perform 25 replicates.
\item \textbf{Experimental Protocol.} All algorithms (with various
  hyperparameter configurations) are run in each environment (with
  different featurization and stochasticity) for 100K episodes. In
  each of the 25 replicates we change the random seed for both the
  environment and the algorithm. We record total reward every 100
  episodes. Since \oracleq and \qlearning operate directly on the
  hidden states, we do not re-run with different featurizations.
\item \textbf{Reported Results.} In Figures~\ref{fig:lock_results}
  and~\ref{fig:lock_find}, we display the \emph{time-to-solve} and
  \emph{time-to-find} for each algorithm. Time-to-solve is the first
  episode $t$ (rounded up to the nearest 100) for which the agent's
  running-average reward is at least $0.5V^\star$. Time to find is the
  first episode $t$ (rounded up to the nearest 100) for which the
  agent has non-zero reward. For central tendency, we plot the median
  of these values over the 25 replicates. For error bars we plot the
  $90^{\textrm{th}}$ and $10^{\textrm{th}}$ percentile. In
  Figure~\ref{fig:lock_structure}, we plot the median running-average
  reward after 100K episodes. There are no error bars in this plot.
\item \textbf{Computing Infrastructure.} All experiments were
  performed on a Linux compute cluster. Relevant software packages and
  versions are: python 3.6.7, numpy 1.14.3, scipy 1.1.0, scikit-learn
  0.19.1, torch 0.4.0, gym 0.10.9, matplotlib 1.5.1.
\end{itemize}

\section{Proofs for Deterministic BMDPs}
\label{sec:proofs_det}

We begin with proofs for deterministic BMDPs as they are simpler and some of the arguments are reused in the more general stochastic case.

The following theorem shows the relation between number of samples and the risk.
Since we are in the deterministic setting, given $(s,a)$, the next hidden state is uniquely determined, we abuse the notation and let $s' = p(s,a)$
\begin{thm}\label{thm:sample_complexity_regression}
	Let $\hvg_h$ be the function defined in line 8 of Algorithm~\ref{algo:deterministic}.
	For any $\epsilon > 0$, if $n = \Omega\left( \frac{1}{\epsilon}\log \left(\frac{\abs{\gclass}}{\delta}\right)\right)$, then with probability at least $1-\delta$ over the training samples $\Dexp$, we have
	\begin{align*}
	\expect_{(s,a) \sim \unif\left(\states_{h-1} \times \actions\right), s' = p(s,a), x' \sim q(\cdot\given s')}\left[\norm{\hvg_h(x')-\vb_\unif(s')}_2^2\right] \le \epsilon.
	\end{align*}
\end{thm}


\begin{proof}[Proof of Theorem~\ref{thm:sample_complexity_regression}]
	The proof is a simple combination of empirical risk minimization analysis and Bernstein's inequality.
	Let $Q_\unif$ denote the joint distribution over $(s,a,s',x')$ such that $(s,a) \sim \unif(\states_{h-1} \times \actions), s' = p(s,a), x'\sim q(\cdot\given s')$.
	Define the population risk as \begin{align*}
	R(\vg) = \expect_{(s,a,s',x')\sim Q_\unif}\left[\norm{\vg(x')-\ve_{(s,a)}}_2^2\right]
	\end{align*}
	We also define the empirical risk as \begin{align*}
	\widehat{R}(\vg) =  \frac{1}{n}\sum_{i=1}^{n} \norm{\vg(x_i')-\ve_{(s_i,a_i)}}_2^2 \triangleq \frac{1}{n}\sum_{i=1}^{n}L_i(\vg).
	\end{align*}
	Note $L_i(\vg) \le 2\left(\norm{\vg(x_i')}_2^2+\norm{\ve_{(s_i,a_i)}}_2^2\right) \le 4$ because $\norm{\vg(x_i')}_2^2 \le \norm{\vg(x_i')}_{1}^2 = 1$ and $\norm{e_{(s_i,a_i)}}_2^2=1$.
	Recall $(s_i,a_i,s_i',x_i') \sim Q_\unif$. Recall that the minimizer $\vg^*$ of $R(\vg)$ satisfies $\vg^*(x') = \vb_\unif(s')$ if $x' \sim q(\cdot\given s')$. We  bound the second moment of the excess risk $L_i(\vg)-L_i(\vg^*)$.
\begin{align*}
	\expect\left[\left(L_i(\vg)-L_i(\vg^*)\right)^2\right] = & \expect\left[\left(
	\norm{\vg(x')-\ve_{(s,a)}}_2^2 - \norm{\vg^*(x')-\ve_{(s,a)}}_2^2
	\right)^2
	\right] \\
	= & \expect\left[\left(
	\left(\vg(x')-\vg^*(x')\right)^\top\left(\vg(x')+\vg^*(x')-2\ve_{(s,a)}\right)
	\right)^2\right] \\
	\le & \expect\left[
	\left(
	\norm{\vg(x')-\vg^*(x')}_2 \norm{\vg(x')+\vg^*(x')-2\ve_{(s,a)}}_2
	\right)^2
	\right] \\
	\le & 16 \expect\left[\norm{\vg(x')-\vg^*(x')}_2^2\right] \\
	= & 16 \left(R(\vg)-R(\vg^*)\right).
\end{align*}
where the expectation is taken over $Q_U$ and the inequality we used $\norm{\vg(x')+\vg^*(x')-2\ve_{(s,a)}}_2 \le \norm{\vg(x')}_2+\norm{\vg^*(x')}_2 + 2\norm{\ve_{(s,a)}}_2 \le \norm{\vg(x')}_{1}+\norm{\vg^*(x')}_{1} + 2\norm{\ve_(s,a)}_{1} = 4$.
Now we apply Bernstein inequality on the random variable $\hat{R}(\vg)-\hat{R}(\vg^*) - (R(\vg)-R(\vg^*)$ and obtain that
 if $n = \Omega\left(\frac{1}{\epsilon}\log\frac{\abs{\gclass}}{\delta}\right)$ we have with probability at least $1-\delta$, for all $\vg \in \cG$
 \begin{align*}
 	\abs{\hat{R}(\vg)-\hat{R}(\vg^*) - (R(\vg)-R(\vg^*))}  \le C \left(\sqrt{\frac{\left(R(\vg)-R(\vg^*)\right)\log(\abs{\cG}/\delta)}{n}} + \frac{\log(\abs{\cG})/\delta)}{n}\right)
 \end{align*} for a universal constant $C >0$.
Note by definition $\hat{g}_h$ satisfies $\hat{R}(g)-\hat{R}(g^*) \le 0$ and $R(\hat{g}_h)-R(g^*) \le 0$ so we have \begin{align*}
	\abs{R(\vg_h)-R(g^*)} \le  C \left(\sqrt{\frac{\left(R(\vg_h)-R(\vg^*)\right)\log(\abs{\cG}/\delta)}{n}} + \frac{\log(\abs{\cG})/\delta)}{n}\right).
\end{align*}
It is easy to see $\abs{R(\vg_h)-R(\vg^*)} = O\left(\frac{\log(\abs{\cG}/\delta)}{n}\right)$.
Plugging in our choice of $n$, we prove the theorem.

%
\end{proof}

\begin{thm}[Restatement of Theorem~\ref{thm:sample_complexity_deterministic}]
Set  $\tau = 0.01$,
$\Nexp = \tOmega(M^2K^2 \log\,\card{\gclass})$
and $\Ncluster = \tOmega(MK)$.
Then with probability at least $1-\delta$, Algorithm~\ref{algo:deterministic}
returns an $\epsilon$--policy cover of $\states$, with $\epsilon=0$.
\end{thm}

\begin{proof}[Proof of Theorem~\ref{thm:sample_complexity_deterministic}]

The proof is by induction on levels.
Our two induction hypothese are \begin{itemize}
\item  For $h'=1,\ldots,h-1$, $\widehat{\states}_{h'}$ and $\states_{h'}$ are bijective, i.e., there exists a bijective function $\alpha: \widehat{\states}_{h'} \rightarrow \states_{h'}$.
\item For $h'=1,\ldots,h-1$, $\Pi_{h'}$ covers the state space with the scale equals to $0$ (c.f. Definition~\ref{defn:cover}).
\end{itemize}
Note these two hypotheses imply Claim~\ref{claim:det}.

To prove the induction, for the base case, this is true because the starting state is fixed.
Now we prove the case $h'=h$.

For simplicity, we let $Q_U$ be a distribution over $(s,a,s',x')$ that $(s,a) \sim \unif\left[\states_{h-1} \times \actions\right], s' = p(s,a), x'\sim q(\cdot\given s')$.
Note because $\states_{h-1}$ and $\widehat{\states}_{h-1}$ are bijective and $Q_U$ can be also viewed as a distribution over $(s,a,s',x')$ that the probability of the event $(\hat{s},a,s',x')$ is the same as the event $(\alpha^{-1}(s),a,s',x')$.
Therefore, notation-wise,  in the proof of this theorem, $s$ is equivalently to $\alpha^{-1}(s)$ and $\hat{s}$ is equivalent to $\alpha(\hat{s})$.

By Theorem~\ref{thm:sample_complexity_regression}, we know if $\Nexp = O\left(\frac{1}{\epsilon} \log\left(\frac{\abs{\gclass}H}{\delta}\right)\right)$, we have with probability at least $1-\frac{\delta}{H}$,  \[\expect_{(s,a,s',x')\sim Q_U}\left[\norm{\vg_h(x')-\vb_\unif(s')}_2^2\right] \le \epsilon.\]
	Recall because of the induction hypothesis and the definition of our exploration policy, we sample $(s,a)$ uniformly,
	we must have for any $s' \in \states_h$\[
	\expect_{ x'\sim q(\cdot\given s')}\left[\norm{\vb_\unif(s')-\vg_h(x')}_{2}^2\right] \le KM\epsilon.
	\]
	By Jensen's inequality, this implies that for $s' \in \states_h$,
	\[
	\norm{\vb_\unif(s')-\expect_{ x'\sim q(\cdot \given s')}\left[\vg_h(x')\right]}_{2}^2 \le KM\epsilon.
	\]
	By AM-GM inequality, we know for any vector $\vv$, $\norm{\vv}_1^2 \le \dim(\vv) \norm{\vv}_2^2$.
	Therefore we have  for any  $s' \in \states_{h}$ \[
	\norm{\vb_{\unif}(s')-\expect_{ x'\sim q(\cdot\given s')}\left[\vg_h(x')\right]}_{1}^2 \le K^2M^2\epsilon
	\]
	Choosing $\epsilon = \frac{\tau^2}{100K^2M^2}$, we have with probability $1-\delta$, for any $h=0,\ldots,H-1$, $s' \in \states_h$, \[
	\norm{\vb_U(s')-\expect_{s' = p(s,a), x'\sim q(\cdot\given s')}\left[\vg_h(x')\right]}_{1} \le \tau/10.
	\]
	The above analysis shows the learned $\hat{\vg}_h$ has small error for all level and all states.

	Next, by standard Hoeffding inequality, we know if $\Ncluster = O\left(\frac{MK\log(MKH/\delta)}{\tau^2}\right)$ with probability at least $1-\frac{\delta}{H}$, for $(s,a) \in \states_{h-1} \times \actions$, we have \[\norm{\vz_{s \odot a} - \expect_{s' = p(s,a), x' \sim q(\cdot \given  s')}\left[\vg_h(x')\right]}_1 \le \tau/10.\]
	
	Now consider an iteration in Algorithm~\ref{algo:find_representation}.
	For any $\pi \odot a$, let $s_\pi$ denotes state reached by following the policy $\pi$ and let $s' = p(s_\pi,a)$.
	If there exists $(\pi',a')$ with $s' = p(s_{\pi'},a')$ and $\vz_{\pi'\odot a'} \in \widehat{\states}_h$, then we know
 $\norm{\vz_{a\circ \pi}-\vz_{a' \circ \pi'}}_1 \le \frac{\tau}{5}$.
	Thus $\hat{z}_{a\circ \pi}$ will not be added to $\widehat{\states}_h$.
	On the other hand, suppose for all $\vz_{\pi' \odot a'} \in \widehat{\states}_{h}$, $s' \neq p(s_{\pi'},a')$.
	Let $s'' = p(s_{\pi'},a')$.
	 We know $\norm{\vz_{\pi \odot a}-\vz_{\pi' \odot a'}}_1 \ge \norm{\vb_{\unif}(s') - \vb_{\unif}(s'')}_1 - \tau /5 \ge 2-\tau/5 \ge\tau $.
	Thus $\vz_{\pi \odot a}$ will be added to $\widehat{\states}_{h}$.
	Note the above reasonings imply $\states_h$ and $\widehat{\states}_h$ are bijective and from Algorithm~\ref{algo:deterministic}, it is clear that for all $s' \in \states_h$ we have stored one path (policy) in $\Pi_h$ that can reach $s'$.
	Thus we prove our induction hypotheses at level $h$.
	Lastly we use union bound over $h=1,\ldots,H$ and finish the proof.

\end{proof}

\begin{cor}[Restatement of Corollary~\ref{cor:pac_deterministic}]
\label{cor:pac_deterministic_app}
With probability at least $1-\delta$, Algorithm~\ref{algo:deterministic} can be used to find an $\epsilon$-suboptimal policy using at most
$\tO\bigParens{M^2K^2H\log\,\card{\gclass} + MKH^3/\epsilon^2}$
trajectories from a deterministic BMDP.
\end{cor}
\begin{proof}[Proof of Corollary~\ref{cor:pac_deterministic}]
For a fixed state action pair $(s,a)$, we collect $\frac{1}{\epsilon^2H^2}\log(\frac{MHK}{\delta})$ samples.
Using Hoeffding inequality, we know with probability at least $1-\frac{\delta}{MHK}$, our estimated $\hat{r}(s,a)$ of this state-action pair satisfies\[
\abs{\hat{r}(s,a)-r(s,a)} \le \frac{1}{H\epsilon}.
\]
Taking union bound over $h \in [H], s \in \states_h, a\in \actions$, we know with probability at least $1-\delta$, for all state-action pair, we have \begin{align}
\abs{\hat{r}(s,a)-r(s,a)} \le \frac{1}{H\epsilon}. \label{eqn:reward_perb}
\end{align}
Now let $(a_1,\ldots,a_H)$ be the sequence of actions that maximizes the total reward based on estimated reward.
Let $\widehat{R}$ be the total estimated reward if we execute $(a_1,\ldots,a_H)$ and let $R$ be the true reward.
By Equation~\eqref{eqn:reward_perb}, we know with probability at least $1-\delta$  over the training samples, we have \[
\abs{\widehat{R}-R} \le \epsilon.
\]
Now denote let $(a_1^*,\ldots,a_H^*)$ be the sequence of actions that maximizes the true total reward .
Let $\widehat{R}^*$ be the total estimated reward if we execute $(a_1^*,\ldots,a_H^*)$ and let $R^*$ be the true reward.
Applying Equation~\eqref{eqn:reward_perb} again, we know
\[
\abs{\widehat{R}^*-R^*} \le \epsilon.
\]
Now note \begin{align*}
	R-R^* = &R-\widehat{R} + \widehat{R} - \widehat{R}^* + \widehat{R}^* - R^* \\
	& \ge R-\widehat{R} + \widehat{R}^* - R^* \\
	& \ge - \abs{R-\widehat{R}} - \abs{ \widehat{R}^* - R^*} \\
	& \ge -2\epsilon.
\end{align*}
Rescaling $\epsilon$ we finish the proof.
\end{proof}

\section{Proof of Theorem~\ref{thm:why_pop_risk_minimizer}}
\label{sec:proof_embedding}
\begin{thm}[Restatement of Theorem~\ref{thm:why_pop_risk_minimizer}]
Let $\nu$ be a distribution supported on $\states_{h-1}\times\actions$ and let $\tilde{\nu}$ be a distribution over $(s,a,x')$ defined by sampling $(s,a)\sim\nu$, $s'\sim p(\cdot\given s,a)$, and $x'\sim q(\cdot\given s')$. Let
\begin{align*}
\vg_h\in\argmin_{\vg \in \gclass}
\expect_{\tilde{\nu}}\Bracks{\norm{\vg(x') - \ve_{(s,a)}}^2}.
\end{align*}
Then, under Assumption~\ref{asmp:realizability}, every minimizer $\vg_h$ satisfies
$\vg_h(x')=\vb_\nu(s')$ for all $x'\in\contexts_{s'}$ and $s'\in\states_h$.
\end{thm}

\begin{proof}[Proof of Theorem~\ref{thm:why_pop_risk_minimizer}]
First, note that $\vb_\nu\left(s'\right)$ is the conditional mean of $(s,a) \in \states_{h-1} \times \actions$ given $s'$. By the optimality of conditional mean in minimizing the least squares loss, we have
\begin{align*}
\vb_\nu(s') = \argmin_{V \in \mathbb{R}^{MK}}\expect_{ (s,a,s') \sim \tilde{\nu}}\left[\norm{V - e_{(s,a)}}_2^2~\mid~s_h=s'\right] .
\end{align*}	
Now we consider the embedding function $\emb(s') = \vb_\nu(s')$. Since $\emb$ is a tabular mapping from $\states_h$ to $\simplex_{MK}$, it minimizes the unconditional squared loss, under any distribution over $s'$. Furthermore, by Assumption~\ref{asmp:realizability}, we know there exists one $\vg_h \in \gclass$ which satisfies that \begin{align*}
	\vg_h(x') = \emb(s') \text{ if } x' \sim q(\cdot\given s').
\end{align*}
Combing these facts we have \begin{align*}
\vg_h\in\argmin_{\vg \in \gclass}
\expect_{\tilde{\nu}}\Bracks{\norm{\vg(x') - \ve_{(s,a)}}^2}
\end{align*}
where we have moved from conditional to unconditional expectations in the last step using the tabular structure of $\phi$ as discussed above. This concludes the proof.
\end{proof}

\section{Justification of Assumption~\ref{asmp:identifiability} and Dependency on $\mu_{\min}$}
\label{sec:justification_pomdp}
The following theorem shows some separability assumption is necessary for exploration methods based on the backward conditional probability representation.
\begin{thm}[Necessary and Sufficient Condition for State Identification Using Distribution over Previous State Action Pair]
	\label{thm:condition_population}
	Fix $h \in \{2,\ldots,H+1\}$ and let $s_1',s_2' \in \states_{h}$.
	\begin{itemize}
		\item Let $U$ denote the uniform distribution over $\states_{h-1} \times \actions$, if the backward probability satisfies $
		\vb_U(s_1') = \vb_U(s_2')
		$, then for any $\nu \in \triangle (\states_{h-1} \times \actions)$ we have \begin{align*}
		\vb_\nu(s_1') = \vb_\nu(s_2').
		\end{align*}
		\item If the transition probability satisfies
		$
		\vb_U(s_1')  \neq \vb_U(s_2')
		$,
		then for any $\nu \in \triangle (\states_{h-1} \times \actions)$ that satisfies $\nu(s,a) > 0$ for any $(s,a) \in \states_{h-1} \times \actions $,we have \begin{align*}
		\vb_\nu(s_1')  \neq \vb_\nu(s_2').
		\end{align*}
	\end{itemize}
\end{thm}
\begin{proof}[Proof of Theorem~\ref{thm:condition_population}]
By Bayes rule, for any $s'\in\states_h$ we have
\[
b_\nu(s, a|s') \propto \, \prob_\nu(s,a) \, p(s'|s,a).
\]
Let $\vp_{s'} := [p(s'\given s,a)]_{(s,a)\in\states_{h-1}\times\actions}$. In matrix form,
\begin{align}
\vb_\nu(s') &~ \propto \diag(\nu) \vp_{s'} \propto \diag(\nu) \, \frac{1}{MK} \,\vp_{s'} \nonumber\\
&~ \propto \diag(\nu) \diag(U) \vp_{s'}
\propto \diag(\nu) \vb_U(s'). \label{eqn:parallel}
\end{align}
When $\vb_U(s_1') = \vb_U(s_2')$, $\diag(\nu) \vb_U(s_1') = \diag(\nu) \vb_U(s_2')$, which implies that $\vb_\nu(s_1') = \vb_\nu(s_2')$. This proves the first claim.

For the second claim, assume towards contradiction that there exists $\nu>0$ such that $\vb_\nu(s_1') = \vb_\nu(s_2')$. From Equation~\eqref{eqn:parallel} we have
\[
\vb_U(s_1') \propto \diag(\nu)^{-1} \vb_\nu(s_1') = \diag(\nu)^{-1} \vb_\nu(s_2') \propto \vb_U(s_2'),
\]
which implies that $\vb_U(s_1') = \vb_U(s_2')$ and contradicts the condition of the claim.
\end{proof}
It shows if the backward probability induced by the uniform distribution over previous state-action pair cannot separate states at the current level, then the backward probability induced by any other distribution cannot do this either.
Therefore, if in Assumption~\ref{asmp:identifiability}, $\gamma = 0$, by Theorem~\ref{thm:condition_population}, there is no way to differentiate $s_1'$ and $s_2'$.

The next lemma shows if there exists a margin induced by the uniform distribution, for any non-degenerate distribution we also have a margin.
\begin{lem}\label{lem:roll_in_margin}
	Let $\nu \in \triangle\left(\states_{h-1}\times \actions\right)$ with $\nu(s,a) \ge \tau$.
	Then under the Assumption~\ref{asmp:identifiability} we have for any $s_1',s_2' \in \states_{h}$\[\norm{\vb_{\nu}(s_1')-\vb_{\nu}(s_2')}_1 \ge \frac{\tau \gamma}{2}.\]
\end{lem}
\begin{proof}[Proof of Lemma~\ref{lem:roll_in_margin}]
	Recall \[
	\vb_{\nu}(s_1') = \frac{\diag(\nu)\vb_U(s_1')}{\norm{\diag(\nu)\vb_U(s_1')}_1}.
	\]
	Therefore we have \begin{align*}
	\norm{\vb_{\nu}(s_1')-\vb_{\nu}(s_2')}_1 = &\norm{
		\frac{\norm{\diag(\nu) \left(
				\vb_U(s_1') - \frac{\norm{\diag(\nu)\vb_U(s_1')}_1}{\norm{\diag(\nu)\vb_U(s_2')}_1}\cdot \vb_U(s_2')
				\right)}_1}{\norm{\diag(\nu)\vb_U(s_1')}_1}
	}_1\\
	\ge & \min_{(s,a)}\nu(s,a) \norm{\left(
		\vb_U(s_1') - \frac{\norm{\diag(\nu)\vb_U(s_1')}_1}{\norm{\diag(\nu)\vb_U(s_2')}_1}\cdot \vb_U(s_2')
		\right)}_1 \\
	\ge & \frac{\tau\gamma}{2}
	\end{align*}
	where the first inequality we used \holder's inequality and the fact that $\norm{\diag(\nu)\vb_U(s_1')}_1 \le 1$ and the second inequality we used Lemma~\ref{lem:dist_mult_bound}.
\end{proof}

The following example shows the inverse dependency on $\mu_{\min}$ is unavoidable.
Consider the following setting.
At level $h-1$, there are two states $\states_{h-1}=\left\{s_1,s_2\right\}$ and there is only one action $\actions = \left\{a\right\}$.
There are two states at level $h = \left\{s_1',s_2'\right\}$.
The transition probability is \begin{align*}
p\left(\cdot|\cdot\right) = \begin{pmatrix}
0.5 & 0.5 \\
0.1 & 0.9
\end{pmatrix}.
\end{align*} where the first row represents $s_1$, the second row represents $s_2$, the first column represents $s_1'$ and the second column represents $s_2'$.
By Theorem~\ref{thm:condition_population}, because the transition probability from $s_1$ to $s_1'$ and $s_2'$, we can only use $s_2$ to differentiate $s_1,s_2'$.
However, if $\mu(s_2) = \exp(-\frac{1}{\epsilon})$, i.e., for all policy, the probability of getting to $s_2$ is exponentially small, then we cannot use $s_2$ for exploration and thus we cannot differentiate $s_1'$ and $s_2'$.

\section{Proof of Theorem~\ref{thm:stochastic_pomdp_sample_complexity} and Claim~\ref{claim:stoc}}
\label{sec:proof_sketch}
We prove the theorem by induction. We first provide a high-level outline of the proof, and then present the technical details. At each level $h \in [H]$, we establish that Claim~\ref{claim:stoc} holds. For convenience, we break up the claim into three conditions corresponding to its different assertions, and establish each in turnup to a small failure probability. The first one is on the learned states and the decoding function.

\begin{condition}[Bijection between learned and true states]\label{cond:bijection}
There exists $\epsdec < \frac12 $ such that there is a bijective mapping $\alpha_h~:~\widehat{\states}_h \to  \states_h$ for which
\begin{align}
\prob_{x \sim q(\cdot\given\alpha_h(\hats))} \left[\hatf_h(x) = \hat{s}\right] \ge 1 - \epsdec.
\label{eqn:one-one}
\end{align}
\end{condition}
%

In words, this condition states that every estimated latent state $\hats$ roughly corresponds to a true latent state $\alpha_h(\hats)$, when we use the decoding function $\hatf_h$. This is because all but an $\epsdec$ fraction of contexts drawn from $\alpha_h(\hats)$ are decoded to their true latent state, and for each latent state $s$, there is a distinct estimated state $\alpha_h^{-1}(s)$ as the map $\alpha_h$ is a bijection.
For simplicity, we define $\vp(s,a) \in \mathbb{R}^{M}$ to be the forward transition distribution over $\states_h$ for $s \in \states_{h-1}$ and $a \in \actions$. We abuse notation to similarly use $\vp(\hat{s},a) \in \mathbb{R}^{M}$ to be the vector $\{\prob(s\given \hats,a)\}_{s \in \states_h}$ of conditional probabilities $\states_h$ for $\hats \in \hatset_{h-1}$ and $a \in \actions$.
Note that unlike $s \in \states_{h-1}$, $\hats \in \hatset_{h-1}$ is not a Markovian state and hence the conditional probability vector $\vp(\hats,a)$ depends on the specific distribution over $\hatset_{h-1}\times \actions$. In the following we will use $\vp^\nu(\hats,a)$ to emphasize this dependency where $\nu$ is the distribution, where $\nu$ is a distribution over $\hatset_{h-1}\times \actions$.

In the proof, we often compare two vectors indexed by $\states_h$ and $\hatset_h$.
We will assume the order of the indices of these two vectors are matched according to $\alpha_h$.

The second condition is on our estimated transition probability.
This condition ensures our estimation has small error.
\begin{condition}[Approximately Correct Transition Probability]\label{cond:trans_prob}
For any $\hats \in \widehat{\states}_{h-1}$, $a \in \actions$,  we have  \[
\norm{\hvp(\hats,a) - \vp(s,a)}_1 \le \epsp \triangleq
\min\left\{\frac{\mu_{\min}\gamma}{10M^3HK},\frac{\epsilon\mu_{\min}}{10H}\right\}.\]
\end{condition}

The following lemma shows if the induction hypotheses hold, then we can prove main theorem.
\begin{lem}\label{lem:from_induction_hypo_to_thm}
	Assume Condition~\ref{cond:bijection} and~\ref{cond:trans_prob} hold for all $h \in [H]$.
	For any $h \in [H]$ and $s \in \states_h$,  there exists $\hats \in \hatset_h$ that the  policy $\pi_{\hats}$ satisfies $\prob^{\pi_{\hat{s}}}(s) \ge \mu(s) -2 H\epsdec-2H\epsp$.
\end{lem}
Based on this lemma, since $\epsdec \le \frac{\epsilon}{3H}$ and $\epsp \le \frac{\epsilon}{3H}$, we prove that the algorithm outputs a policy cover with parameter $\epsilon$, completing the proof of Theorem~\ref{thm:stochastic_pomdp_sample_complexity}.

Note that Condition~\ref{cond:bijection}, Condition~\ref{cond:trans_prob} and Lemma~\ref{lem:from_induction_hypo_to_thm} together imply Claim~\ref{claim:stoc}.

In the rest of this section, we prove focus on establishing that these conditions hold inductively.

\paragraph{Analysis of Base Case $h=1$}
Since by assumption we know we are starting from $s_{1}$ and we set $\hatset_1 = \{s_1\}$,
Conditions~\ref{cond:bijection} and~\ref{cond:trans_prob} directly hold. Note that the transition operator in this case simply corresponds to the degenerate distribution $\vp_1$ with $\vp_1(s_1) = \hvp_1(s_1) = 1$.

Now supposing that the induction hypotheses hold for $h_1 = 1,\ldots,h-1$, we focus on level $h$.
We next show that Conditions~\ref{cond:bijection} and~\ref{cond:trans_prob} hold with probability at least $1-\frac{\delta}{H}$.
This suffices to ensure an overall failure probability of at most $1-\delta$ as asserted in Theorem~\ref{thm:stochastic_pomdp_sample_complexity} via a union bound.

\paragraph{Establishing Condition~\ref{cond:bijection}.} In order to establish the condition, we need to show that our decoding function $\hatf_h$ predicts the underlying latent state correctly almost always. We do this in two steps. Since the functions $\hatf_h$ are derived based on $\hvg_h$ and $\embh_h$, we analyze the properties of these two objects in the following two lemmas. In order to state the first lemma, we need some additional notation. Note that $\eta_h$ and $\hat{f}_{h-1}$ induce a distribution over $\states_{h-1} \times \hatset_{h-1} \times \actions \times \states_h$. We denote this distribution as $\nu_{h}$. With this distribution, we define the conditional backward probability $\hvb_{\nu_h}$: $\states_h \rightarrow\simplex\left(\widehat{\states}_{h-1} \times \actions\right)$ as
\begin{equation}
\hatb_{\nu_h}(\hat{s},a\given s_1') = \frac{p_{h-1}^{\nu_h}(s_1'\given\hat{s},a) \prob^{\nu_h}(\hat{s},a)}{\sum_{\hat{s}_1,a_1} p_{h-1}^{\nu_h}(s_1'\given\hats_1,a_1)\prob^{\nu_h}(\hats,a_1)}.
\label{eqn:back_stoch}
\end{equation}
%
Recall that $\vp^{\nu_h}_{h-1}$ above refers to the distribution over $s_1'$ according the transition dynamics, when $\hats,a$ are induced by $\nu_h$.

With this notation, we have the following lemma.
\begin{lem}
\label{lem:identification_stoch}
  Assume $\epsdec \le \frac{\mu_{\min}^3\gamma}{100M^4K^3}$. Then the distributions $\hatb_{\nu_h}(\hat{s},a|s')$ are well separated for any pair $s_1',s_2' \in \states_h$:
  \begin{align}
    \norm{\hvb_{\nu_h}(s_1') - \hvb_{\nu_h}(s_2')}_1 \ge \frac{\mu_{\min}\gamma}{3MK}.
    \label{eqn:separation_stoch}
  \end{align}
  Furthermore, if $\Nexp = \Omega\left(\frac{M^3K^3}{\epsdec\mu_{\min}^3\gamma^2}\log\left(\frac{\abs{\cG}H}{\delta}\right)\right)$, with probability at least $1-\delta/H$, for every $s' \in \states_{h}$, $\hvg_h$ satisfies \begin{align}
	\prob_{x' \sim q(\cdot\given s')}\left[\norm{\hvg_h(x') - \hvb_{\nu_h}(s')}_{1} \ge \frac{\gamma\mu_{\min}}{100MK}\right] \le \epsdec.
\label{eqn:good_g_stoch}
	\end{align}
\end{lem}

The first part of Lemma~\ref{lem:identification_stoch} tell us that the latent states at level $h$ are well separated if we embed them using $\emb(s') = \hvb_{\nu_h}(s')$ as the state embedding. The second part guarantees that our regression procedure estimates this representation accurately. Together, these assertions imply that any two contexts from the same latent state (up to an $\epsdec$ fraction) are close to each other, while contexts from two different latent states are well-separated. Formally, with probability at least $1-\frac{\delta}{H}$ over the $\Nexp$ training data:
\begin{enumerate}
\item For any $s' \in \states_h$ and $x'_1,x'_2 \sim q(\cdot\given s')$, we have with probability at least $1-2\epsdec$ over the emission process\begin{align}
	\norm{\hvg_h(x'_1)-\hvg_h(x'_2)}_{1} \le \frac{\mu_{\min}\gamma}{50MK}.
\label{eqn:g_close}
\end{align}
\item For any $s_1', s_2' \in \states_h$ such that $s_1' \neq s_2'$, $x_1' \sim q(\cdot\given s_1') $ and $x_2' \sim q(\cdot\given s_2')$, we have with probability at least $1 - 2\epsdec$ over the emission process \begin{align}
	\norm{\hvg_h(x_1')-\hvg_h(x_2')}_{1} \ge \frac{\mu_{\min}\gamma}{4MK}.
\label{eqn:g_far}
\end{align}
\end{enumerate}
In other words, the mapping of contexts, as performed through the functions $\hvg_h$ should be easy to cluster with each cluster roughly corresponding to a true latent state. Our next lemma guarantees that with enough samples for clustering, this is indeed the case.

\begin{lem}[Sample Complexity of the Clustering Step]
	\label{lem:cluster_sample_complexity}
	If $\Ncluster = \Theta\left(\frac{MK}{\mu_{\min}}\log(\frac{MH}{\delta})\right)$ and $\epsdec \le \frac{\delta}{100H\Ncluster}$ we have with probability at least $1-\frac{\delta}{H}$, (1) for every $s' \in \states_h$, there exists at least one point $\vz \in \cZ$ such that $\vz = \hvg_h(x')$ with $x' \sim q(\cdot\given s')$ and $\norm{\hvg_h(x')-\hvb_{\nu_h}(s')}_{1} \le \frac{\mu_{\min}\gamma}{100MK}$ and (2) for every $\vz = \hat{g}_h(x') \in \cZ$ with $x'\sim q(\cdot \given s')$, $\norm{\hvg_h(x')-\hvb_{\nu_h}(s')}_{1} \le \frac{\mu_{\min}\gamma}{100MK}$.
\end{lem}

Based on Lemmas~\ref{lem:identification_stoch} and~\ref{lem:cluster_sample_complexity}, we can establish that Condition~\ref{cond:bijection} holds with high probability.
Note that Condition~\ref{cond:bijection} consists of two parts.
The first part states that there exists a bijective map $\alpha_h: \widehat{\states}_h \rightarrow \states_h$.
The second part states that the decoding error is small.
To prove the first part, we explicitly construct the map $\alpha_h$ and show it is bijective.
We define $\alpha_h: \widehat{\states}_h \rightarrow \states_h$ as\begin{align}
	\alpha_h(\hat{s}') = \argmin_{s \in \states_h} \norm{
		\emb(s') - \embh(\hat{s}')
		}_1 \label{eqn:alpha_defn}
\end{align}
First observe that for any $\hat{s}' \in \widehat{\states}_h$, by the second conclusion of Lemma~\ref{lem:cluster_sample_complexity}, we know there exists $s' \in \states_h$ such that \[
\norm{\embh(\hat{s}')-\emb(s')} \le \frac{\gamma\mu_{\min}}{100MK}.
\]
This also implies for any $s'' \neq s'$, \[
\norm{\embh(\hat{s}') - \emb(s'')} \ge \norm{\emb(s'') - \emb(s'')} - \norm{\embh(\hat{s}')-\emb(s')}  \ge \frac{\gamma\mu_{\min}}{4MK}.
\]
Therefore we know $\alpha_h(\hat{s}') = s'$, i.e., $\alpha_h$ always maps the learned state to the correct original state.

We now prove $\alpha_h$ is injective, i.e., $\alpha(\hat{s}') \neq \alpha_h(\hat{s}'')$ for $\hat{s}' \neq \hat{s}'' \in \widehat{\states}_h$.
Suppose there are $\hat{s}',\hat{s}'' \in \widehat{\states}_h$ such that $\alpha_h(\hat{s}') = \alpha_h(\hat{s}'') = s'$ for some $s' \in \states_h$.
Then using the second conclusion of Lemma~\ref{lem:cluster_sample_complexity}, we know \begin{align*}
\norm{\embh(\hat{s}')-\embh(\hat{s}'') }_1 \le 	\norm{\embh(\hat{s}')-\emb(s') }_1 + 	\norm{\emb(s')-\embh(\hat{s}'') }_1 \le \frac{\gamma\mu_{\min}}{50MK}.
\end{align*}
However, we know by Algorithm~\ref{algo:find_representation}, every $\hat{s}'\neq\hat{s}'' \in \widehat{\states}_h$ must satisfy \[
\norm{\embh(\hat{s}')-\embh(\hat{s}'') }_1 > \tau = \frac{\gamma\mu_{\min}}{30MK}.
\]
This leads to a contradiction and thus $\alpha_h$ is injective.

Next we  prove $\alpha_h$ is surjective, i.e., for every $s' \in \states_h$, there exists $\hat{s}' \in \widehat{\states}_h$ such that $\alpha_h(\hat{s}') = s'$.
The first conclusion in Lemma~\ref{lem:cluster_sample_complexity} guarantees that for each latent state $s' \in \states_h$, there exists $\vz = \hvg(x') \in \cZ$ with $x' \sim q(\cdot\given s')$.
The second conclusion of Lemma~\ref{lem:cluster_sample_complexity} guarantees that \[
\norm{\vz - \emb(s')}_1 \le \frac{\gamma \mu_{\min}}{100MK}.
\]
Now we first assert that all points in a cluster are emitted from the same latent state by combining Equation~\eqref{eqn:separation_stoch}, the second part of Lemma~\ref{lem:cluster_sample_complexity} and our setting of $\tau$. Now the second part of Lemma~\ref{lem:cluster_sample_complexity} implies that there exists $\hat{s}' \in \widehat{\states}_h$ such that $\norm{\vz-\embh(\hat{s}')}_1 \le \frac{\mu_{\min}\gamma}{50MK}$, since $\vz$ and $\embh(\hat{s}')$ correspond to $\hvg$ evaluated on two different contexts in the same cluster.
Therefore we have \[
\norm{\emb(s')-\embh(\hat{s}')}_1 \le \norm{\emb(s')-\vz}_1 +\norm{\vz-\embh(\hat{s}')}_1 \le \frac{\mu_{\min}\gamma}{30MK}
\]
Now we can show that $\alpha_h(\hat{s}') = s'$.
To do this, we show that $\embh(\hat{s}')$ is closer to $\emb(s')$ than the embedding of any state in $\states_h$.
Using the second conclusion of Lemma~\ref{lem:cluster_sample_complexity} and Equation~\ref{eqn:separation_stoch} we know for any $s''\neq s'$\[
\norm{\embh(\hat{s}') - \emb(s'')}_1 \ge \norm{\emb(s') - \emb(s'')}_1 - \norm{\embh(\hat{s}') - \emb(s')}_1 \ge \frac{\gamma\mu_{\min}}{4MK}.
\]
We know $s' = \argmin_{s_1 \in \states_h} \norm{\embh(s_1) - \embh(\hat{s}')}_1$.
Therefore, by the definition of $\alpha_h$ we know $\alpha_h(\hat{s}')=s$.
Now we have finished the proof of the first part of Condition~\ref{cond:bijection}.

For the second part of Condition~\ref{cond:bijection}, note for any $s' \in \states_h$ and $x' \sim q(\cdot \given s')$, by Lemma~\ref{lem:identification_stoch}, we know with probability at least $1-\epsdec$ over the emission process we have \begin{align*}
	\norm{\hvg_h(x') -\emb(s')}_1 \le \frac{\gamma \mu_{\min}}{100MK}.
\end{align*}
For $\hat{s}' = \alpha_h^{-1}(s')$, we have \begin{align*}
	\norm{\hvg_h(x')-\embh(\hat{s}')}_1 \le \norm{\hvg_h(x') -\emb(s')}_1+ \norm{\emb(s') - \embh(\hat{s}')}_1 \le \frac{\gamma\mu_{\min}}{50MK}.
\end{align*}
On the other hand, for $\hat{s}''\in\widehat{\states}_h$ with $\hat{s}''\neq \alpha_h^{-1}(s')$, we have
 \begin{align*}
 \norm{\hvg_h(x')-\embh(\hat{s}'')}_1
 \ge -\norm{\hvg_h(x') -\emb(s')}_1
 + \norm{\emb(s') - \emb(\alpha_h(\hat{s}''))}_1
 - \norm{ \emb(\alpha_h(\hat{s}'')) -\embh(\hat{s}'')}_1\ge \frac{\gamma\mu_{\min}}{4MK}.
 \end{align*}
 Therefore we have with probability at least $1-\epsdec$ \begin{align*}
 	\hatf_h(x') = \argmin_{\hat{s}' \in \widehat{\states}_h}
 	\norm{\embh(\hat{s}') - \hvg_h(x')}_{1} = \alpha_h^{-1}(s'),
 \end{align*}which is equivalent to the second part of Condition~\ref{cond:bijection}.

\paragraph{Establishing Condition~\ref{cond:trans_prob}.}
This part of our analysis is relatively more traditional, as we are effectively estimating a probability distribution from empirical counts in a tabular setting. The only care needed is to correctly handle the decoding errors due to which our count estimates for frequencies have a slight bias. The following lemma guarantees that Condition~\ref{cond:trans_prob} holds.
\begin{lem}\label{lem:trans_prob_sample_complexity}
	If $\epsdec\le \frac{\epsp \mu_{\min}}{10M^2}$ and if $\Np = \Omega\left(\frac{M^2K}{\epsp^2}\log\frac{MHK}{\delta}\right)$, we have that with probability at least $1-\frac{\delta}{H}$ for every $\hat{s} \in \widehat{\states}_{h-1}$, $a \in \actions$\begin{align}
\norm{\hvp(\hats,a) - \vp(\alpha_{h-1}(\hats),a)}_1 \le \epsp,
	\end{align}
\end{lem}

In the following we present proof details.
\subsection{Proof details for Theorem~\ref{thm:stochastic_pomdp_sample_complexity} and Claim~\ref{claim:stoc}}
\label{sec:proof_stoch}
We first define some notations on policies  that will be useful in our analysis.
First, consider a policy $\psi^{true}$ over the true hidden states for $h=1,\ldots,H$:\begin{align*}
\psi^{true}: \states_h \rightarrow \actions,\psi^{true}(s) = a.
\end{align*}
By the one-to-one correspondence between $\hatset_h$ and $\states_h$ ($\hats$ and $\alpha(\hats)$),\footnote{In this following we drop the subscript of $\alpha$ because the one-to-one correspondence is clear.} $\psi^{true}$ also induces a  policy over the learned hidden states
\begin{align*}
\psi^{learned}: \hatset_{h} \rightarrow \actions, \psi^{learned}(\hats) = \psi^{true}(\alpha(\hats)).
\end{align*}
Next, we let $f_1,\ldots,f_{H}$ be the decoding functions for the true states, i.e.,
\begin{align*}
f_h: \contexts_h \rightarrow \states_h, f_h(x) = s \text{ if and only if } x \sim q(\cdot \given S).
\end{align*}
Recall by Condition~\ref{cond:bijection}, we also have approximately correct decoding functions: $\hat{f}_1,\ldots,\hat{f}_{H}$, which satisfy for all $h \in [H]$ and $s \in \states_h$ 
\begin{align*}
	\hat{f}_h: \contexts \rightarrow \hatset_h, \prob_{x \sim q(\cdot \given s)}\left[ \hat{f}_h(x) = \alpha^{-1}(s)\right] \ge 1-\epsdec.
\end{align*}

Now we consider two policies induced by the policies on the hidden states and the decoding function\begin{align*}
\pi^{true}: &\contexts_h \rightarrow \actions \quad
\pi^{true}(x) = \psi^{true}(f_h(x)) \\
\pi^{learned}: &\contexts_h \rightarrow \actions \quad
\pi^{learned}(x) = \psi^{learned}(\hat{f}_h(x))
\end{align*}

The following figure shows the relations among these objects
\begin{align*}
& \psi^{true}: \states_h \to \actions &  \stackrel{f_h}{\Rightarrow} & \qquad \pi^{true}: \contexts_h \rightarrow \actions \\
& \Downarrow \alpha  & & \\
& \psi^{learned}: \hatset_h \to \actions &  \stackrel{\hat f_h}{\Rightarrow} & \qquad \pi^{learned}: \contexts_h \rightarrow \actions
\end{align*}

In our algorithm we maintain estimations of the transition probabilities of the learned states\[\left\{\hp_{h}(\hats'\given \hats,a)\right\}_{h \in [H], \hats \in \hatset_{h-1},a \in \actions, \hats' \in \hatset_{h}}\]
Given these estimated transition probabilities, a policy over the learned hidden states $\psi^{learned}$, and a target learned state $\hats$, we have an estimation of the reaching probability
$\hat{\prob}^{\psi^{learned}}(s)$ which can be computed by dynamic programming as in the standard tabular MDP.
With these notations, we can prove the following useful lemma.

\begin{lem}\label{lem:diff_true_learned}
For any state $\hats\in \hatset_h$, we have \[
\abs{\prob^{\pi^{true}}(\alpha(\hats)) - \prob^{\pi^{learned}}(\alpha(\hats))}  \le 2H\epsdec.
\]
\end{lem}


\begin{proof}[Proof of Lemma~\ref{lem:diff_true_learned}]
Fixing any state $\hats\in \hatset_h$, for any event $\mathcal{E}$ we have
\begin{align*}
&\prob^{\pi^{true}}(\mathcal{E},\alpha(\hat{f}_1(x_1)) =f_1(x_1),\ldots,\alpha(\hat{f}_h(x_h)) =f_h(x_h)) \\
= & \prob^{\pi^{true}} \left(\alpha(\hat{f}_1(x_1)) =f_1(x_1)\right) \prob^{\pi^{true}} \left(\alpha(\hat{f}_2(x_2)) =f_2(x_2)|\alpha(\hat{f}_1(x_1)) =f_1(x_1)\right)\\
&  \cdots \prob^{\pi^{true}} \left(\alpha(\hat{f}_h(x_h)) =f_h(x_h)|\alpha(\hat{f}_1(x_1)) =f_1(x_1),\ldots,\alpha(\hat{f}_{h-1}(x_{h-1})) =f_{h-1}(x_{h-1})\right) \\
 & \cdot\prob^{\pi^{learned}}(\mathcal{E}|\alpha(\hat{f}_1(x_1)) =f_1(x_1),\ldots,\alpha(\hat{f}_h(x_h)) =f_h(x_h)) \\
 = & \prob^{\pi^{learned}} \left(\alpha(\hat{f}_1(x_1)) =f_1(x_1)\right) \prob^{\pi^{learned}} \left(\alpha(\hat{f}_2(x_2)) =f_2(x_2)|\alpha(\hat{f}_1(x_1)) =f_1(x_1)\right)\\
 &  \cdots \prob^{\pi^{learned}} \left(\alpha(\hat{f}_h(x_h) =f_h(x_h)|\alpha(\hat{f}_1(x_1)) =f_1(x_1),\ldots,\alpha(\hat{f}_{h-1}(x_{h-1})) =f_{h-1}(x_{h-1})\right) \\
 & \cdot\prob^{\pi^{learned}}(\mathcal{E}|\alpha(\hat{f}_1(x_1)) =f_1(x_1),\ldots,\alpha(\hat{f}_h(x_h)) =f_h(x_h)) \\
 =&\prob^{\pi^{learned}}(\mathcal{E},\alpha(\hat{f}_{1}(x_1)) =f_{1}(x_1)),\ldots,\alpha(\hat{f}_h(x_h)) =f_h(x_h)).
\end{align*}
because the event $\left\{\alpha(\hat{f}_1(x_1)) =f_1(x_1),\ldots,\alpha(\hat{f}_h(x_h)) =f_h(x_h)\right\}$ happens and under this event $\pi^{true}$ and $\pi^{learned}$ choose the same action at every level so the induced probability distribution is the same.
Now we bound the target error.

\begin{align*}
&\abs{\prob^{\pi^{true}}(\alpha(\hats)) - \prob^{\pi^{learned}}(\alpha(\hats))}  \\
\le &\abs{\prob^{\pi^{true}}(\alpha(\hats)) - \prob^{\pi^{true}}(\alpha(\hats),\alpha(\hat{f}_1(x_1)) =f_1(x_1),\ldots,\alpha(\hat{f}_h(x_h)) =f_h(x_h))} \\
&+ \abs{\prob^{\pi^{true}}(\alpha(\hats),\alpha(\hat{f}_1(x_1)) =f_1(x_1),\ldots,\alpha(\hat{f}_h(x_h)) =f_h(x_h))- \prob^{\pi^{learned}}(\alpha(\hats))}\\
= &\abs{\prob^{\pi^{true}}(\alpha(\hats)) - \prob^{\pi^{true}}(\alpha(\hats),\alpha(\hat{f}_1(x_1)) =f_1(x_1),\ldots,\alpha(\hat{f}_h(x_h)) =f_h(x_h))} \\
&+ \abs{\prob^{\pi^{learned}}(\alpha(\hats),\alpha(\hat{f}_1(x_1)) =f_1(x_1),\ldots,\alpha(\hat{f}_h(x_h)) =f_h(x_h))- \prob^{\pi^{learned}}(\alpha(\hats))}
\end{align*}
To bound the first term, notice that the event $\left\{s_h = \alpha(\hats)\right\}$ is a superset of \[\left\{s_h = \alpha(\hats),
\alpha(\hat{f}_1(x_1)) =f_1(x_1),\ldots,\alpha(\hat{f}_h(x_h)) =f_h(x_h))
\right\}.\] and \begin{align*}
	&\left\{s_h = \alpha(\hats)\right\} \setminus \left\{s_h = \alpha(\hats),
	\alpha(\hat{f}_1(x_1)) =f_1(x_1),\ldots,\alpha(\hat{f}_h(x_h)) =f_h(x_h))
	\right\}\\
	 = & \left\{
	 s_h = \alpha(\hats), \exists h_1 \in [h], \alpha(\hat{f}_{h_1}(x_{h_1}) \neq f_{h_1}(x_{h_1})
	 \right\}
\end{align*}
Therefore, we can bound\begin{align*}
&\prob^{\pi^{true}}(\alpha(\hats)) - \prob^{\pi^{true}}(\alpha(\hats),\alpha(\hat{f}_1(x_1)) =f_1(x_1),\ldots,\alpha(\hat{f}_h(x_h)) =f_h(x_h)) \\
= & \prob^{\pi^{true}}\left(
\alpha(\hats), \exists h_1 \in [h], \alpha(\hat{f}_{h_1}(x_{h_1}) \neq f_{h_1}(x_{h_1})
\right)\\
\le &\prob^{\pi^{true}}\left(
\exists h_1 \in [h], \alpha(\hat{f}_{h_1}(x_{h_1}) \neq f_{h_1}(x_{h_1})
\right) \\
\le & \sum_{h_1=1}^{h}  \prob^{\pi^{true}}(\alpha(\hat{f}_{h_1}(x_{h_1}))\neq f_{h_1}(x_{h_1}))\\
\le &h\epsdec \\
\le &H\epsdec
\end{align*}

Similarly, we can bound \begin{align*}
\abs{\prob^{\pi^{learned}}\left[\alpha(\hats),\alpha(\hat{f}_1(x_1)) =f_1(x_1),\ldots,\alpha(\hat{f}_h(x_h)) =f_h(x_h))\right]  - \prob^{\pi^{learned}}(\alpha(\hats))}\le H\epsdec.
\end{align*}
Combing these two inequalities we have \begin{align} 
\abs{\prob^{\pi^{true}}(\alpha(\hats)) - \prob^{\pi^{learned}}(\alpha(\hats))}   \le 2H\epsdec. \tag*{\qedhere}
\end{align}
\end{proof}

Now we are ready to prove some consequences of this result which will be used in the remainder of the proof.
\begin{lem}[Restatement of Lemma~\ref{lem:from_induction_hypo_to_thm}]
	Assume Conditions~\ref{cond:bijection} and~\ref{cond:trans_prob} hold for all $h \in [H]$.
	For any $h \in [H]$ and $s \in \states_h$,  there exists $\hats \in \hatset_h$ that the  policy $\pi_{\hats}$ satisfies $\prob^{\pi_{\hats}}(s) \ge \mu(s) -2 H\epsdec-2H\epsp$.
\end{lem}
\begin{proof}[Proof of Lemma~\ref{lem:from_induction_hypo_to_thm}]
For any given $s \in \states_h$, by our induction hypothesis, we know there exists $\hats \in \hatset_h$ such that $\alpha(\hats) = s$.

Now we lower bound $\prob^{\pi_{\hats}}(s)$.
First recall $\pi_{\hats}$ is of the form $\psi_{\hats}(\hat{f}_{h_1}(x_{h_1}))$ for $1 \le h_1 \le h-1$, $x_{h_1} \in \contexts$ and $\psi_{\hats}$ maximizes the reaching probability to $\hats$ given estimated transition probabilities.
To facilitate our analysis, we define an auxiliary policy for $h_1 = 1,\ldots
,h-1$\[
\bar{\pi}_{\hats}: \contexts \rightarrow \actions,\bar{\pi}_{\hats}(x_{h_1}) = \psi_{\hats}(\alpha^{-1}(f_h(x_{h_1})))
\]
i.e., we composite $\psi_{\hats}$ with the true decoding function.
We also define $\psi_{\hats} \circ \alpha^{-1}: \states_h \rightarrow \actions$, i.e., this policy acts on the true hidden state that it first maps a true hidden state to the corresponding learned state and then applies policy $\psi_{\hats}$.
Next,
we let $\psi_s: \states_{h_1} \rightarrow \actions$ be the policy that maximizes the reaching probability of $s$ (based on the true transition dynamics) and define \[
\pi_s: \contexts \rightarrow \actions,\pi_s(x_{h_1}) = \psi_{s}(f_h(x_{h_1})).
\]
Note this is the policy that maximizes the reaching probability to $s$.
We also define $\psi_{s} \circ \alpha: \hatset_h \rightarrow \actions$, i.e., this policy acts on the learned hidden state that it first maps a learned hidden state to the corresponding true hidden state and then applies policy $\psi_{s}$.

We will use the following correspondence in conjunction with Lemma~\ref{lem:error_prop} to do the analysis:
\begin{align*}
\states_h &\Leftrightarrow \hatset_h,\\
\vp_h &\Leftrightarrow \hvp_h, \\
\psi_{\hats}\circ \alpha^{-1} &\Leftrightarrow \psi_{\hats}, \\
\psi_s &\Leftrightarrow \psi_s \circ \alpha.
\end{align*}

Now we can lower bound $\prob^{\pi_{\hats}}(s)$.
\begin{align*}
\prob^{\pi_{\hats}}(s) \ge & \prob^{\bar{\pi}_{\hats}}(s) - 2H\epsdec  \tag{Lemma~\ref{lem:diff_true_learned}} \\
= & \prob^{\psi_{\hats}\circ \alpha^{-1}}(s) - 2H\epsdec \tag{definition of $\bar{\pi}^{\hats}$, probability refers to true hidden state dynamics} \\
\ge & \hat{\prob}^{\psi_{\hats}}(\hats) - 2H\epsdec  - H\epsp \tag{Lemma~\ref{lem:error_prop}, probability refers to estimated transition probability}\\
\ge & \hat{\prob}^{\psi_s \circ \alpha}(\hats) - 2H\epsdec  - H\epsp \tag{$\psi_{\hats}$ maximizes the probability to $\hats$ w.r.t.~$\hat{\prob}$}\\
\ge & \prob^{\psi_s}(s) - 2H\epsdec  - 2H\epsp \tag{Lemma~\ref{lem:error_prop}, probability refers to the true hidden state dynamics}  \\
= & \mu(s) - 2H\epsdec -2H\epsp. \tag*{\qedhere}
\end{align*}
\end{proof}

In the following we prove Lemma~\ref{lem:identification_stoch}.
We first collect some basically properties of the exploration policy $\eta_h$.
\begin{lem}\label{lem:min_reaching_prob_exp}
	If $\epsdec \le \frac{\mu_{\min}}{100H}$ and $\epsp \le \frac{\mu_{\min}}{100H}$,  we have $\prob^{\eta_h}(\hats) \ge \frac{\mu_{\min}}{2M}$ for any $\hats \in \hat{\states}_{h-1}$.
\end{lem}
\begin{proof}[Proof of Lemma~\ref{lem:min_reaching_prob_exp}]
	By Lemma~\ref{lem:from_induction_hypo_to_thm} we know $\prob^{\hat{\pi}_{\hats}}(s) \ge \mu_{\min} - 2H\epsdec-2H\epsp$.
	Notice \begin{align*}
	\prob^{\pi_{\hats}}(\hats) \ge & \prob^{\pi_{\hats}}(\hats,s) \\
	\ge & (\mu(s)- 2H\epsdec-2H\epsp)(1-\epsdec)\\
	\ge & (\mu(s)-2H\epsdec-2H\epsp) \cdot 0.99.
	\end{align*}
Since $\eta_h$ uniformly samples from policies $\{\pi_{\hats}\}_{\hats \in \hatset_{h-1}}$, we  have \begin{align*}
\prob^{\eta_h}(\hats) \ge \frac{(\mu(s) -2H\epsdec - 2H\epsp)\cdot 0.99}{M}.
\end{align*}
Lastly, plugging in the assumption on $\epsdec$ and $\epsp$, we prove the lemma.
\end{proof}
\begin{lem} \label{lem:min_reaching_prob_exp_true_states}
	If $\epsdec \le \frac{\mu_{\min}}{100H}$ and $\epsp \le \frac{\mu_{\min}}{100H}$,  we have $\prob^{\eta_h}(s') \ge \frac{\mu(s')}{2MK} \ge \frac{\mu_{\min}}{2MK}$ for any $s' \in \states_h$.
\end{lem}

\begin{proof}[Proof of Lemma~\ref{lem:min_reaching_prob_exp_true_states}]
By Lemma~\ref{lem:from_induction_hypo_to_thm} we know for any $s \in \states_{h-1}$, we have one policy $\pi_{\hats}$ such that $\prob^{\pi_{\hats}}(s) \ge \frac{\mu(s)}{2}$ because $\epsdec$ and $\epsp$ are sufficiently small.
Since for $\eta_h$, we uniformly sample a state $\hats \in \hatset_{h-1}$, we know for all state $s \in \states_{h-1}$, $\prob^{\eta_h} (s) \ge \frac{\mu(s)}{2M}$.
Thus because we uniformly sample actions, we have $\prob^{\eta_h}(s,a) \ge \frac{\mu(s)}{2MK}$ for every $(s,a) \in \states_{h-1} \times \actions$.
Let $\pi_{s'}$ be that policy such that $\prob^{\pi_{s'}} = \mu(s')$.
Note we have \begin{align*}
\prob^{\eta_h}(s') = &\sum_{s \in \states_{h-1}, a\in\actions} (s'\given s,a) \prob^{\eta_{h}}(s,a) \\
= &\sum_{s \in \states_{h-1}, a\in \actions} p(s'\given s,a) \prob^{\pi_{s'}}(s,a) \cdot \frac{\prob^{\eta_{h}}(s,a) }{\prob^{\pi_{s'}}(s,a)} \\
\ge & \sum_{s \in \states_{h-1}, a\in\actions} P(s'\given s,a) \prob^{\pi_{s'}}(s,a) \cdot \frac{\frac{\mu(s)}{2MK}}{\mu(s)} \\
= & \frac{\mu(s')}{2MK}\\
\ge & \frac{\mu_{\min}}{2MK}.
\end{align*}
\end{proof}

Now we ready to prove Lemma~\ref{lem:identification_stoch}.
\begin{lem}[Restatement of Lemma~\ref{lem:identification_stoch}]
  Assume $\epsdec \le \frac{\mu_{\min}^3\gamma}{100M^4K^3}$. Then the distributions $\hatb_{\nu_h}(\hats,a|s')$ are well separated for any pair $s_1',s_2' \in \states_h$:
  \begin{align*}
    \norm{\hvb_{\nu_h}(s_1') - \hvb_{\nu_h}(s_2')}_1 \ge \frac{\mu_{\min}\gamma}{3MK}.
  \end{align*}
  Furthermore, if $\Nexp = \Omega\left(\frac{M^3K^3}{\epsdec\mu_{\min}^3\gamma^2}\log\left(\frac{\abs{\cG}H}{\delta}\right)\right)$ we have with probability at least $1-\delta/H$, for every $s' \in \states_{h}$, $\hvg_h$ satisfies \begin{align*}
	\prob_{x' \sim q(\cdot\given s')}\left[\norm{\hvg_h(x') - \hvb_{\nu_h}(s')}_{1} \ge \frac{\gamma\mu_{\min}}{100MK}\right] \le \epsdec.
	\end{align*}
\end{lem}
\begin{proof}[Proof of Lemma~\ref{lem:identification_stoch}]
We first prove the property on $\hvb_{\nu_h}$.
First by Lemma~\ref{lem:min_reaching_prob_exp} and our definition of $\eta_h$, we know $\prob^{\eta_h}(s,a) \ge \frac{\mu_{\min}}{2MK}$ for any $s \in \states_{h-1}$ and $a \in \actions$.
Recall \[
\vb_{\nu_h}(s,a\given s_1') = \frac{p_{h-1}(s_1'\given s,a) \prob^{\nu_h}(s,a)}{\sum_{s_1,a_1} p_{h-1}(s_1'\given s_1,a_1)\prob^{\nu_h}(s,a_1)}.
\]
Invoking Lemma~\ref{lem:roll_in_margin}, we have for any $s_1',s_2' \in \states_h$ \[
\norm{\vb_{\nu_h}(s_1')- \vb_{\nu_h}(s_2')}_1 \ge \frac{\mu_{\min}\gamma}{2MK}.
\]
Next we show $\norm{\vb_{\nu_h}(s') -\hvb_{\nu_h}(s')}_1 \le \frac{\mu_{\min}\gamma}{6MK}$ for all $s' \in \states_h$.
Note this implies the first part of the lemma.
Consider a vector $\vQ(s') \in \mathbb{R}^{\abs{\states_{h-1} \times \actions}}$ with each entry defined as \begin{align*}
	Q(s')_{(s,a)} = p_{h-1}(s'\given s,a) \prob^{\nu_h}(s,a).
\end{align*}
Similarly we define $\widehat{Q}(s')  \in \mathbb{R}^{\abs{\states_{h-1} \times \actions}}$ with each entry being \begin{align*}
\widehat{Q}(s')_{(\hats,a)} = p_{h-1}^{\nu_h}(s'\given \hats,a) \prob^{\nu_h}(\hats,a).
\end{align*}
It will be convenient to assume that entries in $Q(s')$ and $\widehat{Q}(s')$ are ordered such that the $\widehat{Q}(s')_{(\hats,a)}$ corresponds to $Q(s')_{(\alpha(\hats),a)}$.
 Our strategy is to bound $\|\vQ(s') - \widehat{\vQ}(s')\|_1$, then invoke Lemma~\ref{lem:normalization_perturbation} which gives the perturbation bound on the normalized vectors.
We calculate the point-wise perturbation.
\begin{align*}
&p_{h-1}(s'\given \alpha(\hats),a) \prob^{\nu_h}(\alpha(\hats),a)  - p_{h-1}^{\nu_h}(s'\given \hats,a) \prob^{\nu_h}(\hats,a)\\
 = &\prob(\alpha(\hats),a) \left(
p_{h-1}^{\nu_h}(s'\given \alpha(\hats),a) - p_{h-1}^{\nu_h}(s'\given \hats,a)
\right) + p_{h-1}^{\nu_h}(s'\given \hats,a) \left(
\prob^{\nu_h}(\alpha(\hats),a) - \prob^{\nu_h}(\hats,a)
\right).
\end{align*}
For the second term, we can directly bound \begin{align*}
\abs{\prob^{\nu_h}(\alpha(\hats),a) - \prob^{\nu_h}(\hats,a)} =&\frac{1}{K}\abs{\prob^{\nu_h}(\hats)-\prob^{\nu_h}(\alpha(\hats))} \\
= & \frac{1}{K}\abs{\sum_{s_1 \in \states_{h-1}}\prob^{\nu_h}(\hats,s_1)-\sum_{\hats_1 \in \hatset_{h-1}}\prob^{\nu_h}(\alpha(\hats),\hats_1)}\\
\le & \frac{1}{K}\max\left\{\sum_{s_1 \in \states_{h-1},s_1\neq \alpha(\hats)}\prob^{\nu_h}(\hats,s_1), \sum_{\hats_1 \in \hatset_{h-1},\hats_1 \neq \hats}\prob^{\nu_h}(\alpha(\hats),\hats_1)\right\}
\end{align*}
Note \begin{align*}
\sum_{s_1 \in \states_{h-1},s_1\neq \alpha(\hats)}\prob^{\nu_h}(\hats,s_1) = &\sum_{s_1 \in \states_{h-1}}\prob^{\nu_h}(s_1) \prob_{x \sim q(\cdot\given s_1)} \left[
\hat{f}_h(x) = \hats
\right] \\
\le & \sum_{s_1 \in \states_{h-1}}\prob^{\nu_h}(s_1) \epsdec \\
\le & \epsdec
\end{align*}
where the first inequality we used the induction hypothesis on the decoding error and the second inequality we used $\sum_{s_1 \in \states_{h-1}}\prob^{\nu_h}(s_1)  \le 1$.
Similarly we can bound $\sum_{\hats_1 \in \hatset_{h-1},\hats_1 \neq \hats}\prob^{\nu_h}(\alpha(\hats),\hats_1) \le \epsdec$.
Therefore, we have $\abs{\prob^{\nu_h}(\alpha(\hats),a) - \prob^{\nu_h}(\hats,a)} \le \frac{\epsdec}{K}$.
For the first term, note\begin{align*}
p_{h-1}^{\nu_h}(s'\given \alpha(\hats),a) - p_{h-1}^{\nu_h}(s'\given \hats,a) = \frac{\prob^{\nu_h}(s',\hats,a)}{\prob^{\nu_h}(\hats,a)} - \frac{\prob^{\nu_h}(s',\alpha(\hats),a)}{\prob^{\nu_h}(\alpha(\hats),a)}.
\end{align*}
We already have bound the deviation on the denominator.
\begin{align*}
\abs{\prob^{\nu_h}(s',\hats,a) - \prob^{\nu_h}(s',\alpha(\hats),a)}  = &\abs{\sum_{s_1 \in \states_{h-1}}\prob^{\nu_h}(s',s_1,\hats,a) - \sum_{\hats_1 \in \hatset_{h-1}}\prob^{\nu_h}(s',\hats_1,\alpha(\hats),a)} \\
= &\abs{\sum_{s_1 \in \states_{h-1}, s_1 \neq \alpha(\hats)}\prob^{\nu_h}(s',s_1,\hats,a) - \sum_{\hats_1 \in \hatset_{h-1},\hats_1 \neq \hats }\prob^{\nu_h}(s',\hats_1,\alpha(\hats),a)} \\
\le & \max\left\{\sum_{s_1 \in \states_{h-1}, s_1 \neq \alpha(\hats)}\prob^{\nu_h}(s',s_1,\hats,a) , \sum_{\hats_1 \in \hatset_{h-1},\hats_1 \neq \hats }\prob^{\nu_h}(s',\hats_1,\alpha(\hats),a)\right\} \\
\le & \max\left\{\sum_{s_1 \in \states_{h-1}, s_1 \neq \alpha(\hats)}\prob^{\nu_h}(s_1,\hats,a) , \sum_{\hats_1 \in \hatset_{h-1},\hats_1 \neq \hats }\prob^{\nu_h}(\hats_1,\alpha(\hats),a)\right\} \\
= &  \frac{1}{K}\max\left\{\sum_{s_1 \in \states_{h-1},s_1\neq \alpha(\hats)}\prob^{\nu_h}(\hats,s_1), \sum_{\hats_1 \in \hatset_{h-1},\hats_1 \neq \hats}\prob^{\nu_h}(\alpha(\hats),\hats_1)\right\} \\
\le & \frac{\epsdec}{K}.
\end{align*}

Recall we have $\prob^{\nu_h}(s,a) \ge \frac{\mu_{\min}}{2MK}$, so applying Lemma~\ref{lem:scalar_perturbation} on $\frac{\prob^{\nu_h}(s',\hats,a)}{\prob^{\nu_h}(\hats,a)} - \frac{\prob^{\nu_h}(s',\alpha(\hats),a)}{\prob^{\nu_h}(\alpha(\hats),a)}$, we have \begin{align}
\abs{p_{h-1}^{\nu_h}(s_1'\given \hats,a) - p_{h-1}(s_1'\given \alpha(\hats),a)} \le \frac{4M\epsdec}{\mu_{\min}}. \label{eqn:forward_perb}
\end{align}
Therefore we have
\begin{align*}
\abs{p_{h-1}(s'\given\alpha(\hats),a) \prob^{\nu_h}(\alpha(\hats),a)  - p_{h-1}^{\nu_h}(s'\given\hats,a) \prob^{\nu_h}(\hats,a)} \le \frac{5M\epsdec}{\mu_{\min}}.
\end{align*}
Thus we have \begin{align*}
	\norm{\vQ(s')-\widehat{\vQ}(s')}_1 \le \frac{5M^2K\epsdec}{\mu_{\min}}.
\end{align*}
By Lemma~\ref{lem:min_reaching_prob_exp_true_states}, we know \begin{align*}
	\norm{\vQ(s')}_1=\sum_{(s,a) \in \states_{h-1} \times \actions } p_{h-1}(s'\given s,a) \prob^{\nu_h}(s,a) = \prob^{\eta_h}(s') \ge \frac{\mu_{\min}}{2MK}.
\end{align*}
Therefore applying Lemma~\ref{lem:normalization_perturbation} on $\vQ(s')$ and $\widehat{\vQ}(s')$, we have\begin{align*}
	\norm{\hvb_{\nu_h}(s') - \vb_{\nu_h}(s')}_1 \le \frac{100M^3K^2\epsdec}{\mu_{\min}^2}.
\end{align*}
Since $\epsdec \le \frac{\mu_{\min}^3\gamma}{100M^4K^3}$, it follows that $\norm{\hat{\vb}_{\nu_h}(s') - \vb_{\nu_h}(s')}_1 \le \frac{\mu_{\min}\gamma}{6MK}$.
Note that $\hvb_{\nu_h}(s')$ is a conditional probability, we can apply the same arguments used in proving Theorem~\ref{thm:why_pop_risk_minimizer} to show \begin{align*}
\vg_h(x') = \hvb_{\nu}(s') \text{ for }x' \sim q(\cdot\given s').
\end{align*}
%
%
%
%
%
%
%
%
%

Now we prove the second part of the Theorem about $\hvg_h$.
For simplicity, we set $\epsilon' = \frac{\mu_{\min}^3\gamma^2\epsdec}{20000M^4K^4}$ in the following analysis.
Using the same argument as Theorem~\ref{thm:sample_complexity_deterministic}, since we know $\Nexp = \Omega\left(\frac{M^4K^4}{\epsdec\mu_{\min}^3\gamma^2}\log(\frac{\abs{\cG}}{\delta})\right)= \Omega\left(\frac{1}{\epsilon'}\log\frac{\abs{\cG}}{\delta}\right)$, we have 
\begin{align*}
	\expect_{(\hats,a)\sim \nu_{h}, s' \sim \vp^{\eta_h}(\cdot\given \hats,a),x' \sim q(\cdot\given s')}\left[\norm{\hvg_h(x')-\hvb_{\nu_h}(s')}_2^2\right]  \le \epsilon'
\end{align*}
Therefore, since we know by Lemma~\ref{lem:min_reaching_prob_exp_true_states} for any $s' \in \states_{h}$, $\prob^{\eta_h}(s') \ge \frac{\mu_{\min}}{2MK}$,
 we have for all $s'$\begin{align*}
	\expect_{x' \sim q(\cdot\given s')}\left[\norm{\hvg_h(x')-\hvb_{\nu_h}(s')}_2^2\right]  \le \frac{2MK\epsilon'}{\mu_{\min}}.
\end{align*}
By Markov's inequality, we have
\begin{align*}
\prob_{x' \sim q(\cdot\given s')}\left(\norm{\hvg_h(x')-\hvb_{\nu_h}(s')}_{2}^2 \ge \frac{\gamma^2 \mu_{\min}^2}{10000M^3K^3}\right) \le \frac{20000M^4K^4\epsilon'}{\mu_{\min}^3\gamma^2} \le \epsdec
\end{align*}
Using the fact that $\|\cdot\|_1 \le \sqrt{MK}\|\cdot\|_2$,  we have
\begin{align*}
\prob_{x' \sim q(\cdot\given s')}\left(\norm{\hvg_h(x')-\hvb_{\nu_h}(s')}_{1} \ge \frac{\gamma \mu_{\min}}{100MK}\right) \le \epsdec.
\end{align*}

\end{proof}

\begin{lem}[Restatement of Lemma~\ref{lem:cluster_sample_complexity}]
	If $\Ncluster = \Theta\left(\frac{MK}{\mu_{\min}}\log(\frac{MH}{\delta})\right)$ and $\epsdec \le \frac{\delta}{100H\Ncluster}$ we have with probability at least $1-\frac{\delta}{H}$, (1) for every $s' \in \states_h$, there exists at least one point $\vz \in \cZ$ such that $\vz = \hvg_h(x')$ with $x' \sim q(\cdot\given s')$ and $\norm{\hvg_h(x')-\hvb_{\nu_h}(s')}_{1} \le \frac{\mu_{\min}\gamma}{100MK}$ and (2) for every $\vz = \hat{g}_h(x') \in \cZ$ with $x'\sim q(\cdot \given s')$, $\norm{\hvg_h(x')-\hvb_{\nu_h}(s')}_{1} \le \frac{\mu_{\min}\gamma}{100MK}$.
\end{lem}
\begin{proof}[Proof of Lemma~\ref{lem:cluster_sample_complexity}]
	For any state $s' \in \states_{h}$, by Lemma~\ref{lem:min_reaching_prob_exp_true_states} we know $\prob^{\eta_h}(s') \ge \frac{\mu_{\min}}{2MK}$.
	The probability of not seeing one context generated from this state is upper bounded by $\left(1-\frac{\mu_{\min}}{2MK}\right)^{\Ncluster} \le \frac{\delta}{2MH}$.
	Now taking union bound over $\states_h$, we know with probability at least $1-\frac{\delta}{2H}$, we get one context from every state.
	Furthermore, because we know $\epsdec \le \frac{\delta}{100H\Ncluster}$, by union bound over $\Ncluster$ samples, we know we can decode every context correctly with probability at least $1-\frac{\delta}{2H}$.
\end{proof}


\begin{lem}[Restatement of Lemma~\ref{lem:trans_prob_sample_complexity}]
	If $\epsdec\le \frac{\epsp \mu_{\min}}{10M^2}$ and if $\Np = \Omega\left(\frac{M^2K}{\epsp^2}\log\frac{MHK}{\delta}\right)$, we have that with probability at least $1-\frac{\delta}{H}$ for every $\hat{s} \in \widehat{\states}_{h-1}$, $a \in \actions$\begin{align}
\norm{\hvp(\hats,a) - \vp(\alpha_{h-1}(\hats),a)}_1 \le \epsp,
	\end{align}
\end{lem}
\begin{proof}[Proof of Lemma~\ref{lem:trans_prob_sample_complexity}]
Using Equation~\eqref{eqn:forward_perb}
 and the decoding error bound on $\hat{f}_h$, we know for any $\hats \in \hatset_{h-1}, a \in \actions, \hats' \in \hatset_{h-1}$, we have
\begin{align*}
	\abs{p^{\eta_h}(\hats'\given \hats,a) - p(\alpha_{h}(\hats')\given \alpha_{h-1}(\hats),a)} \le \frac{4M\epsdec}{\mu_{\min}}.
\end{align*}
Summing over $\hatset_h$, we have
\begin{align*}
	\sum_{\hats' \in \hatset_h}\abs{p^{\eta_h}(\hats'\given \hats,a) - p(\alpha_{h}(\hats')\given \alpha_{h-1}(\hats),a)} \le \frac{4M^2\epsdec}{\mu_{\min}}.
\end{align*}
	
Next we bound 	$\norm{\hvp^{\eta_h}(\hats,a)-\vp^{\eta_h}(\hats,a)}_{1}$.
By Lemma~\ref{lem:min_reaching_prob_exp}, we know for every $(\hats,a) \in \hatset_{h-1} \times \actions$, $\prob^{\eta_h}(\hats,a) \ge \frac{\mu_{\min}}{2MK}$.
For each pair, by Theorem~\ref{thm:l1_concentration}, we need $\Omega\left(\frac{M}{\epsp^2}\right)$ samples.
Thus in total we need $N_p = \Omega\left(\frac{M^2K}{\mu_{\min}\epsp^2}\log\frac{MHK}{\delta}\right)$
to make $\norm{\hvp(\hats,a)-\vp^{\eta_h}(\hats,a)}_{1} \le \frac{\epsp}{10}$.
Now combining these two inequalities we have the desired result.
\end{proof}

\section{Technical Lemmas}
\label{sec:technical_lemmas}
\begin{lem}\label{lem:dist_mult_bound}
For any two vectors $u,v\in \mathbb{R}^d_{+}$ with $\norm{u}_1 = \norm{v}=1$ and $\norm{u-v}_1 = \gamma$, we have for any $\alpha > 0$, $\norm{\alpha u - v}_1 \ge \frac{\gamma}{2}$.
\end{lem}
\begin{proof}[Proof of Lemma~\ref{lem:dist_mult_bound}]
Denote $S_+ = \left\{ i \in [d]| u_i > v_i\right\}$ and $S_- = \left\{ i \in [d]| u_i < v_i\right\}$.
Because $\norm{u-v}_1 = \gamma$, we know \[\sum_{i \in S_+} (u_i - v_i) + \sum_{i \in S_-} (v_i-u_i) = \gamma.\]
Also note that 
\[\sum_{i \in S_+} (u_i - v_i) - \sum_{i \in S_-} (v_i-u_i) = \norm{u}_1-\norm{v}_1 = 0.\]
Therefore, \[
\sum_{i \in S_+} (u_i - v_i)  = \sum_{i \in S_-} (v_i-u_i) = \frac{\gamma}{2}.
\]
If $\alpha \ge 1$, we know \begin{align*}
	\norm{\alpha u-v}_1 \ge \sum_{i \in S_+} \alpha u_i - v_i \ge \frac{\gamma}{2}
\end{align*}
and if $\alpha < 1$, we know \begin{align*}
	\norm{\alpha u-v}_1 \ge \sum_{i \in S_+}  v_i -\alpha u_i\ge \frac{\gamma}{2}.
\end{align*}
We finish the proof.
\end{proof}

\begin{lem}\label{lem:error_prop}[Error Propagation Lemma for Tabular MDPs]
Consider two tabular MDPs, $\cM$ and $\widehat{\cM}$.
Let $\states_1,\ldots,\states_H$ be the state space of $\cM$ and $\widehat{\states}_1,\ldots,\widehat{\states}_H$ be the for $\widehat{\cM}$.
The state spaces satisfy that for every $h \in [H]$, $\states_h$ and $\widehat{\states}_h$ are bijective, i.e., there exists a bijective function $\alpha: \widehat{\states}_h \rightarrow \states_h$.
Let $\actions$ be $\cM$ and $\widehat{\cM}$'s shared action space.
For $h=1,\ldots,H$, let $\vp_h$ be the forward operator for $\cM$ and $\hvp_h$ be the forward operator model $\cM$.
For any policy on $\psi: \states_h \rightarrow \actions$ for $\cM$, because the $\states_h$ and $\widehat{\states}_h$ are bijective, $\psi$ induces a policy  for $\widehat{\cM}$, $\hat{\psi}: \widehat{\states}_h \rightarrow \actions$ that satisfies $\psi(\alpha_h(\hat{s})) = \hat{\psi}(\hat{s})$.
Then if \[
\norm{\hvp_h(\hat{s},a)-\vp_h(\alpha(\hat{s}),a)}_1 \le \epsilon
\] for all $h \in [H]$, $a \in \actions$ and $\hat{s} \in \widehat{\states}_h$ (the indices of the vector $\hvp_h(\hat{s},a)$ and $\hvp_h(\alpha(\hat{s}),a)$  are matched according to $\alpha$), we have for any policy $\psi$ for $\cM$, \[
\sum_{s_h \in \states_h}\abs{\hat{\prob}_h^{\hat{\psi}}(\alpha^{-1}(s_h))-\prob_h^\psi(s_h)}  \le h\epsilon
\] 
\end{lem}
\begin{proof}[Proof of Lemma~\ref{lem:error_prop}]
	We prove by induction.
\begin{align*}
	&\sum_{s_h \in \states_h}\abs{\hat{\prob}_h^\psi(\alpha^{-1}(s_h))-\prob_h^\psi(s_h)} \\
	= &\sum_{s_h \in \states_{h}}\abs{\sum_{s_{h-1}\in\states_{h-1}}
		\left(\widehat{\prob}^{\hat{\psi}}(\alpha^{-1}(s_{h-1}))\hp_{h-1}(\alpha^{-1}(s_h)\given \alpha^{-1}(s_{h-1}),\hat{\psi}(\alpha^{-1}(s_{h-1}))) - \prob^\psi( s_{h-1})p(s_h\given s_{h-1},\psi(s_{h-1}))\right) 
	} \\
	\le & \sum_{s_h \in \states_{h}}\abs{\sum_{s_{h-1}\in \states_{h-1}} (\widehat{\prob}^{\hat{\psi}}(\alpha^{-1}(s_{h-1}))-\prob^{\psi}(s_{h-1}))p(s_h\given s_{h-1},\psi(s_{h-1}))} \\
	& + \sum_{s_h \in \states_{h}} \sum_{s_{h-1}\in \states_{h-1}}\widehat{\prob}(\alpha^{-1}(s_{h-1}))\abs{\hp(\alpha^{-1}(s_h)\given\alpha^{-1}(s_{h-1}),\hat{\psi}(\alpha^{-1}(s_{h-1})))-p(s_h\given s_{h-1},\psi(s_{h-1}))}
	\end{align*}
	For the first term,
	\begin{align*}
	&\sum_{s_h \in \states_{h}}\abs{\sum_{s_{h-1}\in \states_{h-1}} (\widehat{\prob}^{\hat{\psi}}(\alpha^{-1}(s_{h-1}))-\prob^{\psi}(s_{h-1}))p(s_h\given s_{h-1},\psi(s_{h-1}))} \\
	\le &\sum_{s_h \in \states_{h}}\sum_{s_{h-1}\in \states_{h-1}} \abs{\widehat{\prob}^{\hat{\psi}}(\alpha^{-1}(s_{h-1})) -\prob^{\psi}(s_{h-1})}p(s_h\given s_{h-1},\psi(s_{h-1})) \\
	= &\sum_{s_{h-1}\in \states_{h-1}}\Bigg( \abs{\widehat{\prob}^{\hat{\psi}}(\alpha^{-1}(s_{h-1})) -\prob^{\psi}(s_{h-1})} \Big(\sum_{s_h \in \states_{h}} p(s_h\given s_{h-1},\psi(s_{h-1}))\Big)\Bigg) \\
	= &\sum_{s_{h-1}\in \states_{h-1}} \abs{\widehat{\prob}^{\hat{\psi}}(\alpha^{-1}(s_{h-1}))-\prob^{\psi}(s_{h-1})}  \tag{transition probabilities sum up to $1$} \\
	\le &(h-1)\epsilon. \tag{induction hypothesis}
	\end{align*}
	For the other term, \begin{align*}
	&\sum_{s_h \in \states_{h}} \sum_{s_{h-1}\in \states_{h-1}}\widehat{\prob}(\alpha^{-1}(s_{h-1}))\abs{\hp(\alpha^{-1}(s_h)\given\alpha^{-1}(s_{h-1}),\hat{\psi}(\alpha^{-1}(s_{h-1})))-p(s_h\given s_{h-1},\psi(s_{h-1}))} \\
	= &\sum_{s_{h-1} \in \states_{h-1}} \widehat{\prob}(\alpha^{-1}(s_{h-1}))\sum_{s_h \in \states_{h}} \abs{\hat{p}(\alpha^{-1}(s_h)\given\alpha^{-1}(s_{h-1}),\hat{\psi}(\alpha^{-1}(s_{h-1})))-p(s_h\given s_{h-1},\psi(s_{h-1}))}\\
	= &\sum_{s_{h-1} \in \states_{h-1}}\widehat{\prob}(\alpha^{-1}(s_{h-1})) \norm{\hvp(\alpha^{-1}(s_{h-1}),\hat{\psi}(\alpha^{-1}(s_{h-1})))-\vp(s_{h-1},\psi(s_{h-1}))}_1 \\
	\le & \sum_{s_{h-1} \in \states_{h-1}}\widehat{\prob}(\alpha^{-1}(s_{h-1})) \, \epsilon
	= \epsilon.
	\end{align*}
	Combining these two inequalities we have the desired result.
\end{proof}

\begin{lem}\label{lem:scalar_perturbation}
	For $a,b,c,d \in \mathbb{R}^+$ with $a\le b$ and $c\le d$, we have \begin{align*}
	\abs{\frac{a}{b}-\frac{c}{d}} \le \frac{\abs{d-b}+\abs{a-c}}{\max\left\{b,d\right\}}.
	\end{align*}
\end{lem}
\begin{proof}[Proof of Lemma~\ref{lem:normalization_perturbation}]
	\begin{align*}
	\abs{\frac{a}{b} - \frac{c}{d}}  = & \abs{\frac{ad-bc}{bd} }\\
	= &\abs{ \frac{ad-ab+ab-bc}{bd} }\\
	= & \abs{\frac{a(d-b)}{bd} + \frac{a-c}{d}} \\
	\le & \frac{\abs{d-b}+\abs{a-c}}{d}.
	\end{align*}
		By symmetry between $b$ and $d$, we obtain the desired result. 
\end{proof}

\begin{lem}\label{lem:normalization_perturbation}
	For any two vector $\vp,\vq \in \mathbb{R}^{d}_+$, we have \begin{align*}
	\norm{\frac{\vp}{\norm{\vp}_1}-\frac{\vq}{\norm{\vq}_1}}_1 \le \frac{2\norm{\vp-\vq}_1}{\max\left\{\norm{\vp}_1,\norm{\vq}_1\right\}}.
	\end{align*}
\end{lem}
\begin{proof}[Proof of Lemma~\ref{lem:normalization_perturbation}]
	\begin{align*}
	\norm{\frac{\vp}{\norm{\vp}_1}-\frac{\vq}{\norm{\vq}_1}}_1 = &\norm{\frac{\vp\norm{\vq}_1-\vq\norm{\vp}_1}{\norm{\vp}_1\norm{\vq}_1}}_1 \\
	= & \norm{\frac{\vp\norm{\vq}_1-\vq\norm{\vq}_1+\vq\norm{\vq}_1-\vq\norm{\vp}_1}{\norm{\vp}_1\norm{\vq}_1}}_1 \\
	\le & \frac{\norm{\vp-\vq}_1}{\norm{\vp}_1} + \frac{\abs{\norm{\vp}_1-\norm{\vq}_1}}{\norm{\vp}_1}\\
	\le & \frac{2\norm{\vp-\vq}_1}{\norm{\vp}_1}.
	\end{align*}
	By symmetry between $\vp$ and $\vq$, we obtain the desired result. 
\end{proof}

\begin{lem}[Perturbation of Point-wise Division Around Uniform Distribution]
	\label{lem:division_prob_perturb}
For any two vector $\vp_1,\vp_2 \in \mathbb{R}_+^d$, we have \begin{align*}
	\norm{ \frac{\vp_1 \oslash \vp_2}{\norm{\vp_1 \oslash \vp_2}_1}- \begin{pmatrix}
		1/d \\
		\ldots \\
		1/d
		\end{pmatrix}}_1 \le \frac{2  \norm{\vp_1-\vp_2}_1}{d\min_{s} \vp_2(s)}.
\end{align*} where $\oslash$ denotes pointwise division.
\end{lem}
\begin{proof}[Proof of Lemma~\ref{lem:division_prob_perturb}]
Let $d' = \norm{\vp_1\oslash \vp_2}_1$ and $\vect{1}$ be the all one vector of dimension $d$.
\begin{align*}
	\norm{\frac{\vp_1 \oslash \vp_2}{\norm{\vp_1 \oslash \vp_2}_1}- \begin{pmatrix}
	1/d \\
	\ldots \\
	1/d
	\end{pmatrix}}_1 = & \norm{\frac{d \vp_1\oslash p_2 - d' \vect{1}}{d'd}}_1 \\
\le &  \norm{\frac{d \vp_1\oslash \vp_2 - d' \vp_1\oslash \vp_2}{dd'}}_1 + \norm{\frac{d' \vp_1\oslash \vp_2 - d' \vect{1}}{d'd}}_1 \\
= &  \abs{d-d'}\frac{\norm{\vp_1\oslash \vp_2}_1}{dd'} + \norm{\frac{\vp_1\oslash \vp_2 - \vect{1}}{d}}_1 \\
= & \frac{\abs{d-d'}}{d} + \norm{\frac{\vp_1\oslash \vp_2 - \vect{1}}{d}}_1.
\end{align*}
Note for any $s \in [d]$, we have \begin{align*}
	\abs{\frac{p_1(s)}{p_2(s)} - 1} = \frac{\abs{p_1(s)-p_2(s)}}{p_2(s)} \le \frac{\abs{p_1(s)-p_2(s)}}{\min_{s_1 \in [d]}p_2(s_1)}.
\end{align*}
Therefore, we have \begin{align*}
	\norm{\frac{\vp_1\oslash \vp_2 - \vect{1}}{d}}_1 \le \frac{\sum_s\abs{p_1(s)-p_2(s)}}{d\min_{s}p_2(s)} = \frac{\norm{\vp_1-\vp_2}_1}{d\min_s p_2(s)}.
\end{align*}
Also note that, $\frac{|d' - d|}{d} = \frac{|\norm{\vp_1\oslash \vp_2}_1 - \|\vect{1}\|_1|}{d}\le	\norm{\frac{\vp_1\oslash \vp_2 - \vect{1}}{d}}_1\le  \frac{\norm{\vp_1-\vp_2}_1}{d\min_s p_2(s)}$.
Plugging in these two bounds we obtain our desired result.
\end{proof}

\begin{lem}[Conditional Probability Perturbation Around Uniform Distribution]\label{lem:conditoinal_prob_perturb}
	Let $\vp_1,\vp_2 \in \mathbb{R}_+^d$ with $\norm{\vp_1}_1 = \norm{\vp_2}_1 = 1 $ and $\vp_2 = \left(1/d,\ldots,1/d\right)^\top$.
	Then for any $\vq \in \mathbb{R}_+^d$ we have \begin{align*}
	\norm{\frac{\vq\odot \vp_1}{\vq^\top \vp_1} -\frac{\vq\odot \vp_2}{\vq^\top \vp_2} }_1 \le 2d \norm{\vp_1-\vp_2}_1
	\end{align*} 
	where $\odot$ represents point-wise product.
\end{lem}
\begin{proof}[Proof of Lemma~\ref{lem:conditoinal_prob_perturb}]
	Note the left hand size is independent of the scale of $\vq$, so without loss of generality we assume $\norm{\vq}_1 = 1$.
	We calculate the quantity of interest.
	\begin{align*}
	\frac{\vq\odot \vp_1}{\vq^\top \vp_1} -\frac{\vq\odot \vp_2}{\vq^\top \vp_2} = 
	&\frac{\vq \odot \vp_1(\vq^\top \vp_2)-\vq\odot \vp_2 (\vq^\top \vp_1)}{\vq^\top \vp_1 \cdot
		\vq^\top \vp_2} \\
	=&\frac{\vq\odot \vp_1(\vq^\top \vp_2 -\vq^\top \vp_1)+\vq^\top \vp_1(\vq\odot \vp_1 - \vq\odot \vp_2)}{\vq^\top \vp_1 \cdot \vq^\top \vp_2}.
	\end{align*}
	By \holder inequality, we have $\abs{\vq^\top \vp_1 - \vq^\top \vp_2} \le \norm{\vq}_{\infty} \norm{\vp_1-\vp_2}_1$.
	Furthermore, note $\norm{\vq \odot \vp_1}_1 = \vq^\top \vp_1$ because of the positivity and $\vq^\top \vp_2 = \frac{1}{d}$ because $\vp_2$ is a uniform distribution. 
	Now we can bound \begin{align*}
	\norm{\frac{\vq\odot \vp_1(\vq^\top \vp_2 -\vq^\top \vp_1)}{\vq^\top \vp_1 \cdot \vq^\top \vp_2}}_1 \le  \frac{\norm{\vq}_{\infty}\norm{\vp_1-\vp_2}_1}{1/d} \le d \norm{\vp_1-\vp_2}_1.
	\end{align*}
	
	Next, apply \holder inequality again, we have $\norm{\vq\odot \vp_1 - \vq\odot \vp_2}_1 \le \norm{\vq}_{\infty}\norm{\vp_1-\vp_2}_1$.
	Therefore we can bound\begin{align*}
	\norm{\frac{\vq^\top \vp_1(\vq\odot \vp_1 - \vq\odot \vp_2)}{\vq^\top \vp_1 \cdot \vq^\top \vp_2}}_1 \le \frac{\norm{\vq}_\infty\norm{\vp_1-\vp_2}_1}{1/d} \le d\norm{\vp_1-\vp_2}_1. \tag*{\qedhere}
	\end{align*}
\end{proof}

\section{Concentration Inequalities}
\label{sec:concentration}

\begin{thm}[$L_1$ distance concentration bound (Theorem 2.2 of \cite{weissman2003inequalities})]
	\label{thm:l1_concentration}
	Let $p$ be a distribution over $\mathcal{A}$ with $\abs{\mathcal{A}}=a$.
	Let $X_1,\ldots,X_m \sim p$ and $\hat{p}_{X^m}$ be the empirical distribution.
	Then we have \begin{align*}
	\prob\left(\norm{p-\hat{p}_{X^m}}_1 \ge \epsilon\right) \le \left(2^a-2\right)\exp\left(-\frac{m\epsilon^2}{8}\right).
	\end{align*}
\end{thm}
A directly corollary is the following sample complexity.
\begin{cor}\label{cor:l1_concentration}
	if we have $m \ge 8\left(\frac{a}{\epsilon^2}\log\frac{1}{\delta}\right)$ samples, then with probability at least $1-\delta$, we have $\norm{p-\hat{p}_{X^m}}_1 \ge \epsilon$.
\end{cor}


%
%

\end{document}